\documentclass[letterpaper]{article}
\usepackage{uai2019}
\usepackage[margin=1in]{geometry}

\usepackage{times}
\usepackage[round]{natbib}
\usepackage{soul}
\usepackage{url}
\usepackage[hidelinks]{hyperref}
\usepackage[utf8]{inputenc}
\usepackage[small]{caption}
\usepackage{graphicx}
\usepackage{amsmath}
\usepackage{booktabs} %
\usepackage{xcolor}
\usepackage{subcaption}
\usepackage{algorithm}
\usepackage{algorithmicx}
\usepackage[noend]{algpseudocode}
\usepackage{cleveref}
\usepackage{multicol}
\usepackage{soul}
\usepackage{pgf,tikz,pgfplots}
\pgfplotsset{compat=1.14}
\usepackage{mathrsfs}
\usetikzlibrary{arrows}
\usepackage{tkz-euclide}
\usetkzobj{all}
\usepackage{fancybox}
\usepackage{amsthm}
\usepackage{dsfont}
\usepackage{stmaryrd}

\newcommand{\dif}[1]{\mathrm{d}{#1}}
\newcommand{\lra}{\mathcal{LRA}}
\newcommand{\la}{\mathcal{LA}}

\newcommand{\B}{\mathds{B}}
\newcommand{\R}{\mathds{R}}
\newcommand{\mathL}{\mathcal{L}}
\newcommand{\mathP}{\mathcal{P}}

\newcommand{\mi}{\text{MI}}
\newcommand{\vv}[1]{\boldsymbol{#1}}
\newcommand{\true}[0]{\texttt{true}}
\newcommand{\false}[0]{\texttt{false}}

\newcommand{\price}{\mathit{price}}
\newcommand{\squarefeet}{\mathit{sqft}}
\newcommand{\offset}{\mathit{offset}}

\newcommand{\shatter}{\mathit{shatter}}

\DeclareMathOperator{\WMI}{WMI}
\DeclareMathOperator{\MI}{MI}

\DeclareMathOperator{\PENODE}{PE\_NODE}
\DeclareMathOperator{\PEEDGE}{PE\_EDGE}

\newenvironment{shrinkeq}[1]
{\bgroup
\addtolength\abovedisplayshortskip{#1}
\addtolength\abovedisplayskip{#1}
\addtolength\belowdisplayshortskip{#1}
\addtolength\belowdisplayskip{#1}}
{\egroup\ignorespacesafterend}

\newtheorem{thm}{Theorem}[section]
\newtheorem{cor}[thm]{Corollary}
\newtheorem{mydef}[thm]{Definition}

\newtheorem{cla}[thm]{Claim} 
\newtheorem{exa}[thm]{Example} 
 
\newtheorem{pro}[thm]{Proposition}

\newcommand{\zhe}[1]{}
\newcommand{\guy}[1]{}

\title{
    Efficient Search-Based Weighted Model Integration
}

\author{} %

\author{ {\textbf{Zhe Zeng} \textnormal{~and~} \textbf{Guy Van den Broeck} } \\
 Computer Science Department\\
University of California, Los Angeles\\
\texttt{\{zhezeng, guyvdb\}@cs.ucla.edu} \\
}

\hypersetup{draft}
\begin{document}

\maketitle

\begin{abstract}
  Weighted model integration (WMI) extends weighted model counting to integration in mixed discrete-continuous domains. It has shown tremendous promise for solving probabilistic inference problems in graphical models and probabilistic programs. Yet, state-of-the-art tools for WMI have limited performance and ignore the independence structure that is crucial to improving efficiency. To address this limitation, we propose an efficient model integration algorithm for theories with tree primal graphs. We exploit the sparse graph structure by using search to performing integration. Our algorithm greatly improves the computational efficiency on such problems and exploits context-specific independence between variables. Experimental results show dramatic speedups compared to existing WMI solvers on problems with tree-shaped dependencies.
\end{abstract}

\section{INTRODUCTION}

Weighted model counting (WMC) is the task of counting the weighted sum of all satisfying assignments of 
a propositional logic theory.
In recent years, WMC was shown to be an effective solution for addressing probabilistic inference in a wide spectrum of formalisms \citep{sang2005performing,chakraborty2014distribution,ermon2013embed,chavira2008probabilistic,choi2013compiling,VdBFTDB17,fierens2015inference}.

An inherent limitation of WMC is that it can only deal with discrete distributions.
In order to overcome this restriction,
weighted model integration (WMI) \citep{belle2015probabilistic} was introduced as a generalization of WMC towards hybrid domains, characterized by both discrete and continuous variables.
The formalism relies on satisfiability modulo theory (SMT) \citep{barrett2018satisfiability} technology,
which permits reasoning about the satisfiability of 
theories involving, for example, linear constraints over reals.
WMI works by summing a simple weight function over solutions to Boolean variables and integrating over solutions to the real variables of an SMT theory.
Weight functions play the role of (unnormalized) densities, whereas the logic theory captures the structure of the distribution.
WMI (or closely related formulations) has recently
been applied to several probabilistic graphical model
and programming tasks
\citep{chistikov2015approximate,Albarghouthi:2017:FPV:3152284.3133904,morettin2017efficient,belle2017weighted,braz2016probabilistic}.

Both WMI and WMC are sum-of-product problems~\citep{bacchus2009solving}. In discrete domains, such problems are amenable to a divide-and-conquer approach called search-based inference, where variables are instantiated recursively until the inference problem decomposes.
Solving WMC by search, exploiting problem-specific structure, has been shown to be highly effective, in particular on graphical models that exhibit sparsity \citep{chavira2008probabilistic}.
However, progress in WMI is far from its Boolean counterpart, and currently does not exploit independence. More generally, exact inference algorithms for hybrid graphical models do not exploit sparsity and structure as much as discrete graphical model inference algorithms.

As a first approach to leverage structure, in this paper, we propose a search-based inference procedure for exact model integration that leverages decomposition to speed up inference.
We demonstrate how local structure encoded in SMT theories gives rise to context-specific decomposition during search, reducing the number of models to be generated and integrated over.
The integration problem is decomposed into sub-problems by instantiating shared variables and recursing independently on the resulting simplified SMT theories.
We show how to choose finitely many values to instantiate continuous variables with, and subsequently do polynomial interpolation to recover exact answers to WMI problems.
Our complexity analysis proves the first tractability result for a non-trivial class of WMI problems.
Moreover, our experimental evaluation confirms that the approach is drastically faster than existing alternatives on WMI problems with sparse, tree-shaped primal graphs.

\section{BACKGROUND}

We assume that the reader is familiar with propositional logic and the SAT problem~\citep{biere2009handbook}.
Model counting (\#SAT) is the task of counting the number solutions (models) to a given SAT problem~\citep{gomes2009model}.
Weighted model counting (WMC) generalizes this task by summing weights associated with individual SAT solutions. It is widely used as tool for probabilistic reasoning~\citep{sang2005performing,chavira2008probabilistic,ermon2013embed,chakraborty2014distribution,fierens2015inference,VdBFTDB17}.

Satisfiability Modulo Theories (SMT) generalizes SAT to determining the satisfiability of a formula with respect to a decidable background theory.
In particular, we will consider quantifier-free SMT formulas in the theory of linear arithmetic over the reals, or SMT($\lra$). Here, formulas are Boolean combinations of atomic propositions (e.g., $a$, $b$), and of atomic $\lra$ formulas over real variables (e.g., $x < y + 5$). 
Variable instantiations are denoted as $b^\star$ or $x^\star$.
Sets are denoted in boldface.

\begin{exa}\label{exa: price}
    For a house $i$, let $\price_i$ be its price and $\squarefeet_i$ its square footage.
    We can build a simple SMT($\lra$) formula of the relationship between these real variables, with the corresponding solution space depicted in Figure~\ref{fig: price}. That is, SMT($\lra$) formula $\gamma_i$ is
    \begin{equation*}
    \resizebox{1.00\hsize}{!}{$
      \begin{array}{lr}
        (\price_i < 10 \cdot \squarefeet_i + 1000) \lor ~ (\price_i < 20 \cdot \squarefeet_i + 100) \\[0.1cm]
        (0 < \price_i < 3000) \land (0 < \squarefeet_i < 200).
      \end{array}$}
    \end{equation*}
\end{exa}

\begin{figure}[!t]
\hspace{0.9cm}
\begin{tikzpicture}
\begin{axis}[
    width=0.35\textwidth,
    height=0.16\textwidth,
    xlabel={Square Footage},
    ylabel={Price},
    tick label style={font=\scriptsize},
    label style={font=\footnotesize}]
\addplot+[fill, color=green!70!black, mark options={green!90!black}] coordinates
{(0,0) (0,1000) (90,1900) (145,3000) (200,3000) (200,0)} \closedcycle;
\end{axis}
\end{tikzpicture}
\vspace{-0.3cm}
\caption{Feasible region of SMT theory $\gamma_i$ from Example~\ref{exa: price} \label{fig: price}}
\end{figure}
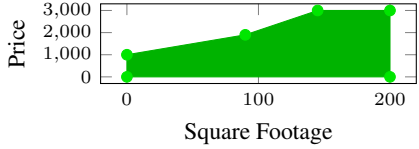

Weighted model integration~(WMI) generalizes WMC to support SMT($\lra$) formulas and real variables \citep{belle2015probabilistic}.
In its simplest form, model integration~(MI) or \#SMT \citep{chistikov2015approximate} computes the volume of the solution space. For example, the green area in Figure~\ref{fig: price} is $430{,}250$.
General WMI is defined as follows~\citep{belle2015probabilistic,morettin2017efficient}.
\begin{mydef} \label{def: wmi}
    Suppose we have 
        $n$ real variables $\vv{x}$, 
        $m$ Boolean variables $\vv{b}$, 
        an SMT($\lra$) formula $\theta(\vv{x}, \vv{b})$, ranging over $\vv{x}$ and $\vv{b}$,
        and a weight function $w(\vv{x}, \vv{b})$ that maps variable instantiations to real weights.
    Then, weighted model integration (WMI)  computes
\begin{shrinkeq}{0ex}
\begin{equation*}
  \WMI(\theta, w \mid \vv{x}, \vv{b}) = 
  \sum\nolimits_{\vv{b^\star} \in \B^m} 
  \int_{\theta(\vv{x}, \vv{b^\star})} 
  w(\vv{x}, \vv{b^\star}) 
  \, d \vv{x}.
\end{equation*}
\end{shrinkeq}
\end{mydef}

That is, the WMI is obtained by summing over every instantiation (total truth assignment) $\vv{b^\star}$ to the Boolean variables,
and integrating $w(\vv{x}, \vv{b}^*)$ over the set of solutions $\{ \vv{x}^* \mid \theta(\vv{x}^*, \vv{b^\star}) \text{ is SAT} \}$.

Weight functions $w$ are usually defined as products of literal weights \citep{belle2015probabilistic,chavira2008probabilistic}.
That is, for some set of literals $\mathL$ we are given a set of per-literal weight functions $\mathP = \{ p_{\ell}(\vv{x}) \}_{\ell \in \mathL}$. When literal $\ell$ is satisfied in a world, denoted $\vv{x} \land \vv{b} \models \ell$, that world's weight is multiplied by $p_{\ell}(\vv{x})$. Formally,
\begin{shrinkeq}{-1.5ex}
\begin{equation*}
    w(\vv{x}, \vv{b}) = \prod\nolimits_{\substack{\ell \in \mathL, \vv{x} \land \vv{b} \models \ell}} p_{\ell}(\vv{x}). 
\end{equation*}    
\end{shrinkeq}
When all variables are Boolean (i.e., $\vv{x} = \emptyset$), the per-literal weights $p_{\ell}(\vv{x})$ are constants and we retrieve the original definition of WMC as a special case of WMI \citep{chavira2008probabilistic}.
In this paper, we assume that all per-literal weights are polynomials.
This setting is expressive enough to approximate any continuous distribution \citep{belle2015probabilistic}. 

\begin{exa}\label{exa: weight}
Consider a formula $(b \lor \neg b) \land \gamma_i$ where $b$ is a Boolean variable and $\gamma_i$ is as defined in Example \ref{exa: price}.
Consider the set of literals $\mathL = \{b, (0 \! < \! \price_i \! < \! 3000)\}$ and per-literal weight functions $\mathP = \{p_{b}, p_{(0 < \price_i < 3000)} \}$, with $p_b(\vv{x}) = 1.5$ and $p_{(0 < \price_i < 3000)}(\vv{x}) = \price_i^2$.
Then, in worlds where both literals in $\mathL$ are satisfied, our weight function is
\begin{shrinkeq}{-1.5ex}
\begin{equation*}
    \resizebox{1\hsize}{!}{$
     p_{b}(\price_i, \squarefeet_i) \cdot p_{(0 < \price_i < 3000)}(\price_i, \squarefeet_i)  = 1.5 \cdot \price_i^2. $}
\end{equation*}  
\end{shrinkeq}
In worlds where $b$ is false and only $(0 < \price_i < 3000)$ is satisfied, the weight function is $\price_i^2$.
\end{exa}
Moreover, we will show that this class of weight functions is well-behaved. In particular, it allows for a natural reduction to unweighted model integration and is amenable to efficient integration.

WMI was introduced as a tool for hybrid probabilistic reasoning. Indeed, the weight of each world can be interpreted as an unnormalized density, and the WMI is its partition function subject to the logical constraints.
Under these semantics, suppose that we are interested in the probability of query $q = \left(\price_i < 2000\right)$ in house price model $\gamma_i$. That probability can be computed as the ratio of two WMI problems: $\Pr(q) = \WMI(\gamma_i \land q) / \WMI(\gamma_i) = 350{,}250/430{,}250 = 81.4\%$.

\begin{figure}[!t]
    \centering
    \includegraphics[width=0.48\textwidth]{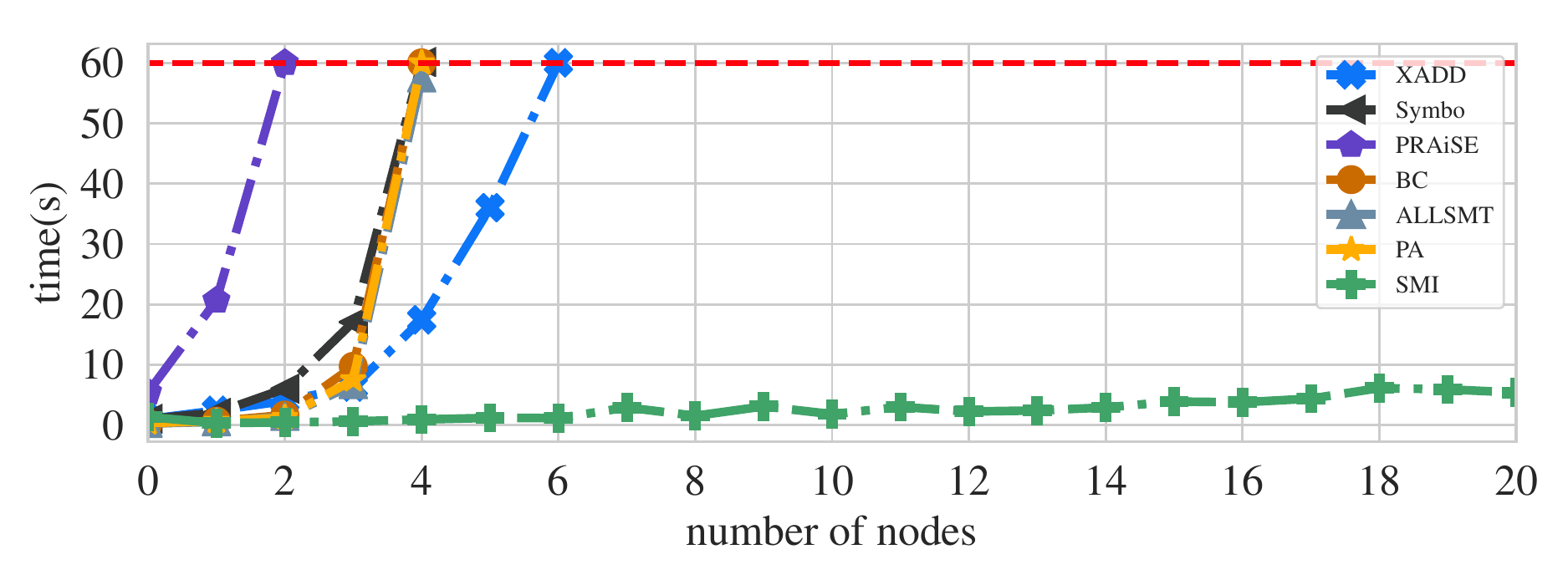}
    \caption{WMI runtime on independent model in Example~\ref{exa: indept house}.}
    \label{fig: runtime-indep}
\end{figure}
    
\paragraph{Exact WMI Solvers}

The first solver for exact WMI \citep{belle2015probabilistic} was a proof-of-concept relying on a simple block-clause strategy~(BC). It iteratively generates new models of a Boolean abstraction of the SMT formula. Each model individually is easily integrated using tools such as \textsc{Latte}~\citep{baldoni2011integrate,de2013software}.
\citet{belle2016component} proposed an all-satisfying-assignments-based solver (ALLSMT).
Unfortunately, enumerating models of the SMT abstraction is prohibitive in practice -- there are exponentially many models, and enumerating them does not exploit  structural properties of the SMT theory such as independence.
Improvements to this algorithm include predicate-abstraction solvers \citep{belle2016component, morettin2017efficient} (PA) and knowledge-compilation solvers \citep{kolb2018efficient} (XADD) and Symbo \citep{zuidberg2019exact}.
The PRAiSE solver \citep{braz2016probabilistic} performs search-based inference on literals of SMT models (not theory variables) and can also be used to solve WMI problems.
Nevertheless, WMI solvers come with no tractability guarantees and still enumerate Boolean models even when there is abundant independence structure, as we will show next.

\section{STRUCTURE IN WMI PROBLEMS}

This section shows how to reduce WMI to model integration (MI) problems whose structural independence properties can be captured by graph abstractions. 

\subsection{INDEPENDENCE}

We begin by motivating why we want to exploit independence structure during probabilistic reasoning.
\begin{exa}\label{exa: indept house}
  Consider $n$ houses, and conjoin the theory $\gamma_i$ from Example \ref{exa: price} $n$ times, once for each house, into a larger SMT theory $\gamma = \land_{i=1}^n \gamma_i$.
  The $n$ houses are independent since no formula in $\gamma$ connects properties of different houses. Thus, the WMI of $\gamma$ can be computed by multiplying the WMI of each individual theory $\gamma_i$. \\
  Figure~\ref{fig: runtime-indep} takes a trivial weight function and compares existing WMI solvers on this simple problem. None is able to exploit the extreme independence structure in $\gamma$. Our proposed method SMI, however, runs in linear time, as expected by the trivial factorization.
\end{exa}

This explosion in runtime is due to the fact that existing solvers ignore independence between variables in the SMT($\lra$) theory.
However, in discrete graphical models and WMC, leveraging independence to decompose problems is at the core of all exact inference methods, and search-based algorithms in particular~\citep{darwiche2009modeling,dechter2007and}.
Specifically, exact discrete inference methods \emph{create} independence even when it is not immediately present, by performing a case analysis on selected discrete variables, instantiating them to all values, and simplifying the model. Through this process, search-based inference algorithms induce and exploit context-specific independence~\citep{boutilier1996context}.
The decompositions afforded by (conditional and context-specific) independence vastly reduce the computational cost of inference. Example \ref{exa: indept house} illustrated that this intuition carries over to WMI problems.

In what follows, we first describe the graph abstraction of SMT theories that characterizes dependencies between variables. These form the basis of our algorithm.
Second, we show how WMI in hybrid domains can be reduced to unweighted MI in real domains. Hence, the solver we develop in this paper will target MI problems.

\subsection{GRAPH ABSTRACTION OF SMT}

Primal graphs are often used to characterize variable dependencies.
For the example Boolean CNF formula $\theta_B = (y \lor x_1) \land (y \lor x_2)$ the primal graph is shown in Figure~\ref{fig: small primal}. Its edges encode that variable pairs $(y,x_1)$ and $(y,x_2)$ appear in the same clause, while $(x_1,x_2)$ never appear together, and are thus independent given $y$.
Similarly, we will use primal graphs for SMT theories to capture variable dependency information as follows.

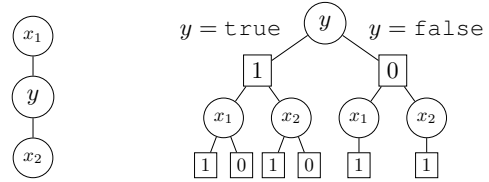
\begin{figure}[tb]
\begin{subfigure}[t]{0.15\textwidth}      
    \centering
    \scalebox{0.9}{
    \begin{tikzpicture}[level/.style={sibling distance=5mm/#1},
        level distance=9mm,]
    \node [circle,draw, scale=0.8] (y){$x_1$}  
      child {node [circle, draw](x){$y$} 
        child {node [circle, draw, scale=0.8](z){$x_2$} }
      };
    \end{tikzpicture}
    }
    \caption{Primal graph}
    \label{fig: small primal}
  \end{subfigure}
  ~
  \begin{subfigure}[t]{0.3\textwidth}       
  \centering
  \scalebox{0.9}{
  \begin{tikzpicture}[
    level/.style={sibling distance=20mm/#1},
    level distance=7.0mm,
    level 3/.style={sibling distance=5.3mm}
    ]
  \node [circle,draw] (y){$y$}  
    child {node [rectangle, draw](t){$1$} 
      child {node [circle, draw, scale=0.8](x1){$x_1$} 
        child {node [rectangle, draw, scale=0.8] {$1$}}
        child {node [rectangle, draw, scale=0.8] {$0$}}
      }
      child {node [circle, draw, scale=0.8](x2){$x_2$} 
        child {node [rectangle, draw, scale=0.8] {$1$}}
        child {node [rectangle, draw, scale=0.8] {$0$}}
      }
      edge from parent node[above left]{$y = \true$}
    }
    child {node [rectangle, draw](f){$0$}
      child {node [circle, draw, scale=0.8](x1){$x_1$} 
        child {node [rectangle, draw, scale=0.8] {$1$}}
      }
      child {node [circle, draw, scale=0.8](x2){$x_2$} 
        child {node [rectangle, draw, scale=0.8] {$1$}}
      }
      edge from parent node[above right]{$y = \false$}
    };    
  \end{tikzpicture}
  }
  \caption{Discrete And/Or Search Tree}
  \label{fig: discrete aos tree}
\end{subfigure}
\caption{Primal graph and search tree for $(y \lor x_1) \land (y \lor x_2)$.}
\end{figure}

\begin{mydef}{\textbf{(Primal graph of SMT)}}\label{def:smt primal}
  The primal graph of an SMT($\lra$) CNF is an undirected graph whose vertices are all variables and whose edges connect any two variables that appear in the same clause.
\end{mydef}
\begin{exa}
  \label{ex:ig}
  Consider the following 
  theory $\theta_n$. 
\begin{shrinkeq}{-1.5ex}
  \begin{equation*}
  \theta_n = 
  \left\{
  \begin{array}{lr}
  (-1 \leq y \leq 1) \land (-0.5 \leq x_1, \cdots, x_n \leq 0.5) \\
  (x_i + 1 \leq y) \lor ( y \leq x_i - 1), \text{ for all }i \in [n] \\
  \end{array}
  \right.
  \end{equation*}  
\end{shrinkeq}
  Figure \ref{fig: theta 2} shows its primal graph and solution space.
\begin{figure}[tb]
   \centering
\begin{subfigure}[t]{0.22\textwidth} 
  \scalebox{0.9}{
    \begin{tikzpicture}[level/.style={sibling distance=10mm/#1},level distance=12mm,]
    \node [circle,draw] (y){$y$}  
        child {node [circle, draw, scale=0.8](x1){$x_1$}} 
        child {node [circle, draw, scale=0.8](x2){$x_2$}}
        child [missing]
        child {node [circle, draw, scale=0.8](xn){$x_n$}
      };
    \path (x2) -- (xn) node [midway] {$\cdots$};
    \end{tikzpicture}}
  \caption{Primal graph of theory $\theta_n$}
  \label{fig: primal graph}
\end{subfigure}
~
\begin{subfigure}[t]{0.22\textwidth}
    \centering
    \includegraphics[width=0.95\textwidth]{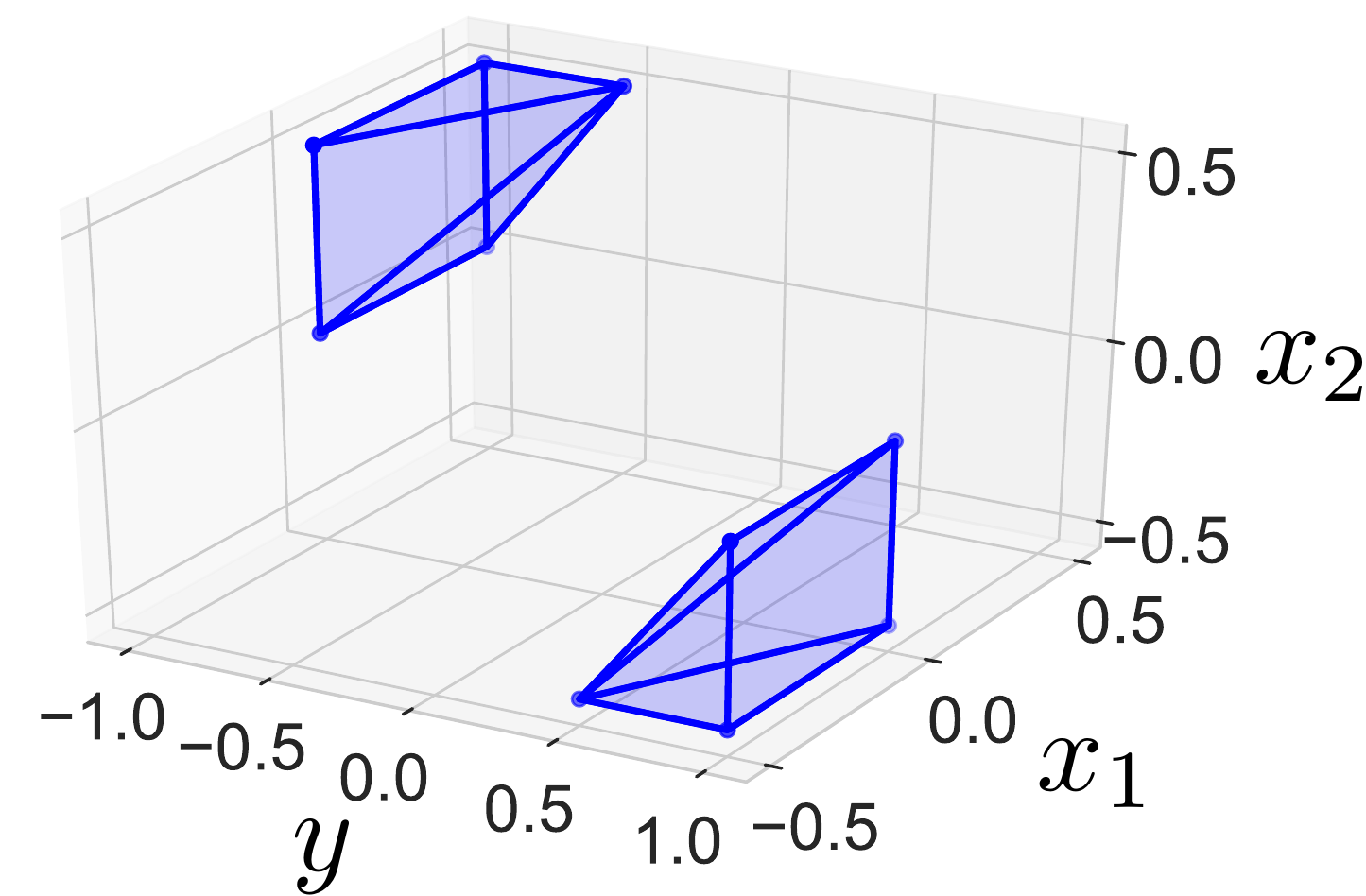}
    \caption{Feasible region for $\theta_2$ }
    \label{fig: model space}
\end{subfigure}
\caption{Primal graph and feasible region from Example \ref{ex:ig}.}
\label{fig: theta 2}
\end{figure}
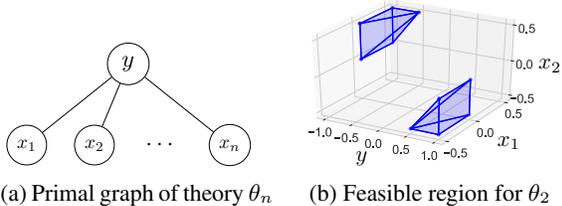
\end{exa}

While there are many flavors of \emph{search-based} exact inference, including recursive conditioning~\citep{darwiche2001recursive}, DPLL model counting~\citep{sang2005performing}, knowledge compilation~\citep{chavira2008probabilistic}, and SumProd algorithms~\citep{bacchus2009solving}, we use the And/Or-search framework to illustrate the required concepts \citep{nilsson1982principles,dechter2007and}.

The And/Or search algorithm for WMC problems recursively simplifies a discrete counting problem by alternating between two steps. 
The first (OR) step selects a Boolean variable and tries to instantiate it to both true and false (we will later see how to choose the variable).
The second (AND) step finds ways of partitioning the WMC problem into independent sub-problems that can be solved separately.
Such sub-problems are introduced by instantiating variables in the OR step in a way that creates independence.
The OR step is also called the Shannon expansion. 
The AND step is also referred to as component caching~\citep{sang2005performing} or detecting decomposability~\citep{chavira2008probabilistic}.

This process is illustrated in Figure~\ref{fig: discrete aos tree} for the earlier Boolean CNF $\theta_B$. Circles denote OR-step variables whose square-node children are its instantiations. After instatiating $y$, the search tree creates independent problems for $x_1$ and $x_2$. This independence can be read off directly from the primal graph in Figure~\ref{fig: small primal}.
Search-based algorithms (with caching) are known to run efficiently on WMC problems with a tree or tree-like primal graph \citep{darwiche2009modeling,bacchus2009solving}.

\subsection{MODEL INTEGRATION IS ALL YOU NEED}
\label{sec: MI is all u need}

\guy{proposition that under certain conditions this does not add cycles to the primal graph}

\guy{Boolean random variables with weights are also easily simulated by a single new real value, and two disjoint interval literals of length proportionate to the desired weight. }

This section casts hybrid WMI problems into MI problems over only real variables. We consider the case where per-literal weight functions are monomials -- functions of the form $\beta x_1^{\alpha_1} \cdots x_n^{\alpha_n}$ over real variables $x_i$ where $\beta \in \R$ and $\alpha_i \in \mathds{N}$.
We further assume that literals in $\mathcal{L}$ also appear in the theory, and that literals and their weights range over the same real variables.

We first show that any WMI problem with Boolean variables can be reduced to a WMI problem without Booleans.
Then we show that WMI problems with per-literal weights can be reduced to an unweighted MI problem where the weight function is~$1$.

\begin{pro}\label{pro: b to nb wmi}
 For each problem $\WMI(\theta, w \mid \vv{x}, \vv{b})$ there exists an equivalent problem $\WMI(\theta^\prime, w^\prime \mid \vv{x}^\prime)$ without Boolean variables $\vv{b}$ such that
  \begin{shrinkeq}{-1ex}
  \begin{equation*}
    \WMI(\theta, w \mid \vv{x}, \vv{b}) = \WMI(\theta^\prime, w^\prime \mid \vv{x}^\prime)
  \end{equation*}
  \end{shrinkeq}
  and the primal graphs of $\theta$ and $\theta^\prime$ are isomorphic.
\end{pro}
This reduction encodes Boolean variables using fresh real variables and replaces each Boolean atom and its negation by two exclusive $\lra$ atoms over those real variables.
Proposition~\ref{pro: b to nb wmi} allows us to focus on WMI problems without Boolean variables involved. Certain weight functions can also be reduced, as we show next.

\begin{pro}\label{pro: wmi to mi}
  For each problem $\WMI(\theta, w \mid \vv{x})$ with per-literal weights $w$ as defined in this section,
  there exists an equivalent unweighted problem $\MI(\theta^\prime \mid \vv{x}^\prime)$ s.t.\
  \begin{shrinkeq}{-1ex}
  \begin{equation*}
    \WMI(\theta, w \mid \vv{x}) = \MI(\theta^\prime \mid \vv{x}^\prime).
  \end{equation*}
  \end{shrinkeq}
  Moreover, 
  when weights $w$ are defined over univariate literals,
  theories $\theta$ and $\theta^\prime$ have identical primal graph treewidth~\citep{Robertson1986}.
  \footnote{An earlier version of this paper omitted the requirement that the per-literal weights need to be defined over univariate literal for treewidth to remain the same after reduction.}
  
\end{pro}
This reduction encodes weights using auxiliary parameter variables. For each literal over which a weight is defined, two set of clauses will be appended such that if the literal holds, the MI over the auxiliary variables equals the monomial weight function; otherwise, it equals 1.

Crucially, both reductions can be constructed in polynomial time. Similar efficient reductions exist for arbitrary polynomial weight functions, but can slightly increase treewidth.
Detailed descriptions of these reduction processes are included in Appendix~\ref{sec: MI with poly weight}.

\begin{exa}\label{exa: weighted price}
    Consider SMT($\lra$) theory $(b \lor \neg b) \land \gamma_i$ with literal set $\mathL$ and per-literal weight functions $\mathP$ as defined in Example \ref{exa: weight}.
    There exists an equivalent MI problem $\MI(\delta \mid \vv{x} \cup \{\lambda_b, z_b, z_i^{(1)}, z_i^{(2)} \})$ with a weight function of $1$ and without Boolean variables.
    Its SMT($\lra$) theory $\delta$ is shown below. Note that its primal graph remains a tree.
  \begin{shrinkeq}{-1.5ex}
  \begin{equation*}
    \resizebox{1.0\hsize}{!}{$
    \delta = 
    \left\{
    \begin{array}{l}
        \gamma_i \land (-1 < \lambda_b < 1) \land_{j = 1, 2} ~ (0 < z_i^{(j)} < \price_i)\\ 
         \phantom{\neg(}\lambda_b > 0\phantom{)} \Rightarrow (0 < z_b < 1.5) \\
         \neg (\lambda_b > 0) \Rightarrow (0 < z_b < 1).
    \end{array}
    \right.$}
    \end{equation*}
\end{shrinkeq}
\end{exa}

\section{SEARCH-BASED MI}

The goal of our work is to take advantage of the independence structure in SMT($\lra$) theories to reduce the computational cost of model integration.
Our solution is to exploit context-specific independence by search.

One obstacle is that to introduce independence in discrete search, we instantiate a variable with all values in its domain.
Unfortunately, when the variable has a real domain (e.g., $y \in [0, 1]$), we cannot instantiate it with every value in its domain, since there are uncountably many (see Figure~\ref{fig: failed aos tree}). This basic limitation has precluded the use of search-based inference in continuous graphical models.

\begin{figure}[!tp]
  \begin{subfigure}[t]{0.2\textwidth}             
  \centering
  \raisebox{19pt}{\scalebox{0.85}{
  \begin{tikzpicture}[level/.style={sibling distance=6mm/#1, align=center},
    level distance=14mm]
  \node [circle,draw] (y){$y$}  
    child {node [rectangle, draw, minimum height=6.5mm](t){$0$}      
      child {node[scale=0.8, below=-8mm] {$\vdots$\vspace{210mm}}}
    }
    child {node [rectangle, draw](t){$\frac{1}{4}$}       
      child {node[scale=0.8, below=-8mm] {$\vdots$}}
    }
    child {node [rectangle, draw](t){$\frac{1}{2}$}          
      child {node[scale=0.8, below=-8mm] {$\vdots$}}
    }
    child [missing]
    child {node [](t){infinitely \\ many values}      
      child {node[scale=0.8, below=-8mm] {$\vdots$}}
    }
    child [missing]
    child {node [rectangle, draw, minimum height=6.5mm](t){$1$}     
      child {node[scale=0.8, below=-8mm] {$\vdots$}}
    };    
  \end{tikzpicture}
  }}
  \caption{Infinite Search Tree}
  \label{fig: failed aos tree}
\end{subfigure}
 ~
\begin{subfigure}[t]{0.25\textwidth}
\centering
\scalebox{0.9}{
\begin{tikzpicture}[
    level 1/.style={sibling distance=14mm},
    level 2/.style={sibling distance=7mm},
    level 3/.style={sibling distance=3mm},
    level distance=8mm,]
\node [circle,draw] (z){$y$}  
  child {node [rectangle,draw] (f) {$\frac{5}{8}$} 
    child {node [circle,draw, scale=0.8] (g) {$x_1$}
      child {node[scale=0.8] {$\vdots$}}
    }
    child {node [circle,draw, scale=0.8] (h) {$x_2$}
      child {node[scale=0.8] {$\vdots$}}
    }
    edge from parent node[above left]{$p(y \mid_{y = \frac{5}{8}})$}
  }
  child {node [rectangle,draw, below=2mm] (a) {$\frac{3}{4}$}
    child {node [circle,draw, scale=0.8] (b) {$x_1$}
      child {node [rectangle,draw, below=-2mm, scale=0.8] (c) {
        \small{
        $
          \int_{-\frac{1}{2}}^{\frac{3}{4} - 1} 1 \dif x_1  
        $
        }
        }}
      child [missing]
      child [missing]
      child [missing]
    }
    child {node [circle,draw, scale=0.8] (d) {$x_2$}      
      child [missing]
      child [missing]
      child [missing]
      child {node [rectangle,draw, below=-2mm, scale=0.8] (e) {
        \small{
        $
          \int_{-\frac{1}{2}}^{\frac{3}{4} - 1} 1 \dif x_2
        $
        }
        }}      
    }
  edge from parent node[below=-2.6mm] {$~~~p(y \mid_{y = \frac{3}{4}})$}
  }
  child {node [rectangle,draw] (i) {$\frac{7}{8}$}
    child {node [circle,draw, scale=0.8] (j) {$x_1$}
      child {node[scale=0.8] {$\vdots$}}
    }
    child {node [circle,draw, scale=0.8] (k) {$x_2$}
      child {node[scale=0.8] {$\vdots$}}
    }
    edge from parent node[above right]{$p(y \mid_{y = \frac{7}{8}})$}
  }
;
\end{tikzpicture}
}
\caption{Our Finite Search Tree}
\label{fig: cont aos tree}
\end{subfigure}
\caption{
Continuous search trees for $\theta_2$ from Example~\ref{ex:ig}.}
\end{figure}
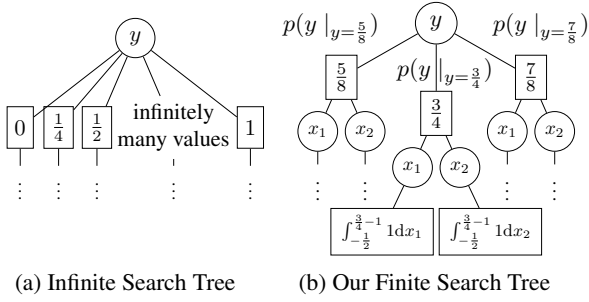

We overcome this problem by observing that MI is an integration over a piecewise polynomial, which can be fully recovered from a finite number of points.
Specifically,
for real variable $y$ in theory $\theta$,
if we instantiate the variable $y$ with a value $\alpha$,
then the MI of theory $\theta \land (y = \alpha)$ is the density of $\WMI(\theta, w)$ at $y = \alpha$.
Recall that a polynomial function $p(y)$ with degree $d$ defined over an interval $I$ is uniquely defined by its values at $d+1$ distinct points in $I$, and that
a closed-form expression for $p(y)$ 
can be recovered exactly and efficiently.

Consider again the theory $\gamma_i$ from Example \ref{exa: price}. As shown in Figure~\ref{fig: price}, function $f(\alpha) = \MI(\gamma_i \land (\squarefeet_i = \alpha))$ is a piecewise polynomial with three intervals. We can recover all three polynomials from a finite number of points, and thus obtain the integration of $f(\alpha)$, that is, the model integration $\MI(\gamma_i)$.
This motivates the search-based model integration algorithm we develop next.

\subsection{VARIABLE INSTANTIATION}
\label{sec: var ini}
We first show that
when per-literal weight functions $\mathP$ are polynomials,
WMI of theory $\theta$ can be obtained 
by doing search with finite instantiations on real variables.
\begin{pro}\label{pro: piece poly}
  Let $y$ be a real variable in SMT($\lra$) theory $\theta$ with a tree primal graph.
  If per-literal weight functions $\mathP$ are polynomials,
  the WMI is an integration over a univariate piecewise polynomial $p(y)$, that is,
  \begin{shrinkeq}{-1.5ex}
  \begin{align}\label{equ: wmi px}
  \WMI(\theta, w \mid \vv{x}, \vv{b}) = \int_{I} p(y) \dif y  
\end{align}
\end{shrinkeq}
where piecewise polynomial $p(y)$ is integrated over set $I = \{ y^* \mid \exists \vv{\hat{x}^*}, \exists \vv{b}^* ~s.t.~ \theta(y^*, \vv{\hat{x}^*}, \vv{b}^*) \text{ is SAT} \}$ with $\vv{\hat{x}}$ being the remaining real variables.
\end{pro}

The set $I$ is a union of disjoint supports for piecewise polynomial $p(y)$.
We refer to these intervals as ``pieces''.
To describe our MI algorithm, we first assume in this section that these intervals and their polynomial degrees are given.
Our method to explicitly find these intervals and degrees will be given in Section~\ref{sec: Piece Enum}.

Although Proposition~\ref{pro: piece poly} holds for WMI problems with polynomial per-literal weight functions in general, we use the insights from Section~\ref{sec: MI is all u need} to only focus on MI problems.
For interval set $I$ defined in Proposition \ref{pro: piece poly}, suppose we are given the interval pieces $[l, u] \in I$ and degrees $d$ of their associated polynomials. If we instantiate variable $y$ with
$d+1$ distinct values 
in each piece $[l,u]$ of degree $d$,
and solve any sub-problems recursively,
we can recover polynomial $p_{l,u}(y)$ defined on interval $[l, u]$
by performing interpolation on $d+1$ points.
Finally, MI of the full theory $\theta$ can be computed as follows.
\begin{shrinkeq}{-1.5ex}
\begin{align}\label{eq: wmi from sum}
  \MI(\theta, w \mid \vv{x}, \vv{b}) = \sum\nolimits_{[l,u] \in I} \int_l^u p_{l,u}(y) \dif y.
\end{align}
\end{shrinkeq}

\begin{figure}[tb]
\centering
\includegraphics[width=0.35\textwidth]{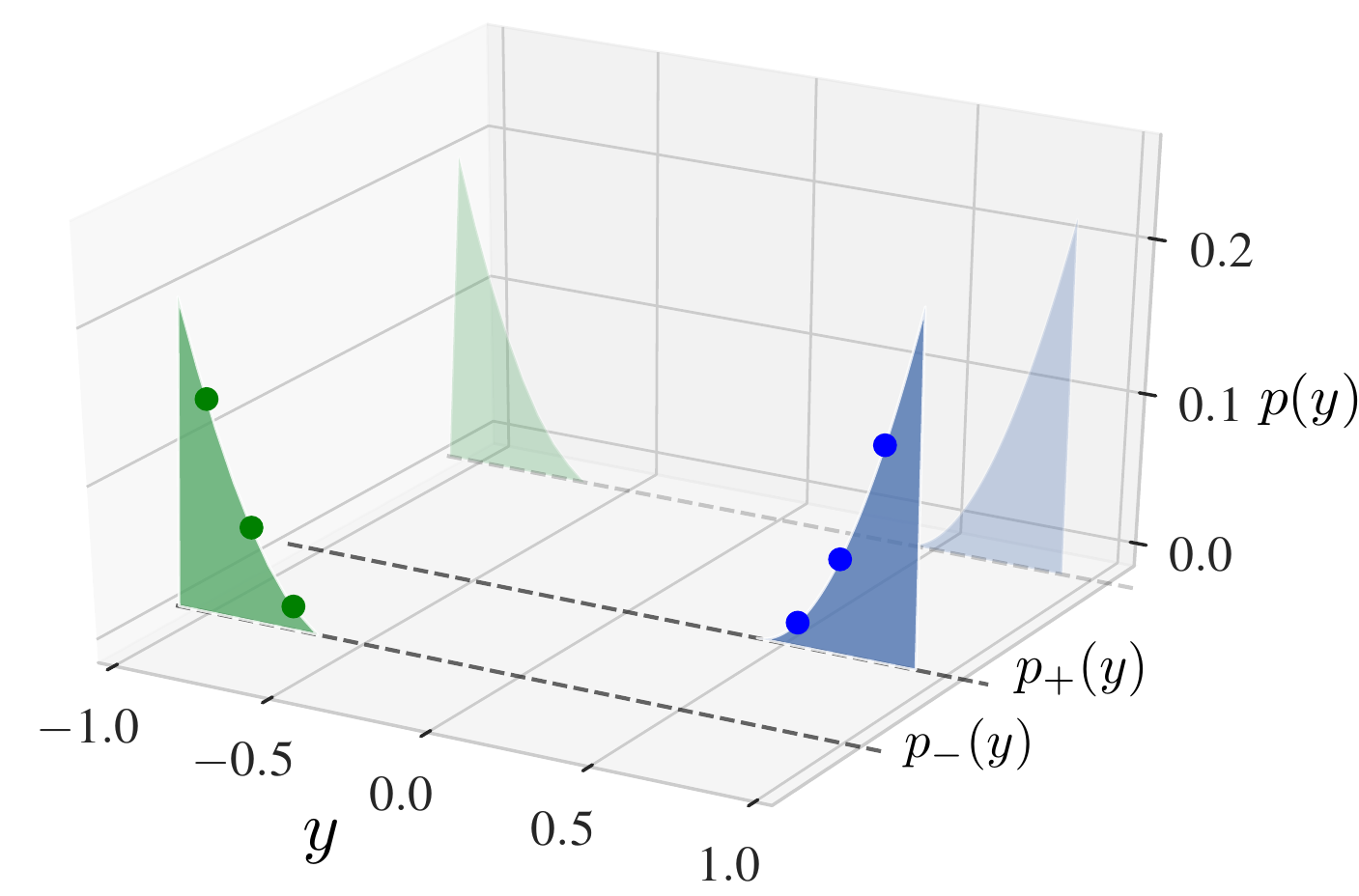}
\caption{Piecewise polynomial $p(y)$ as defined in Proposition~\ref{pro: piece poly} for theory $\theta_2$ from Example~\ref{ex:ig}, whose integration is $\MI(\theta_2)$. The two polynomials $p_{-}(y)$ and $p_{+}(y)$ are unknown, but we can recover them from a finite number of points.}
\label{fig: p(y) curve}
\end{figure}

For example, consider theory $\theta_2$ from Example \ref{ex:ig}. 
We can interpret $\MI(\theta_2)$ as an integration over piecewise polynomial $p(y)$ whose intervals $[-1, -0.5]$  and $[0.5, 1]$ both have associated degree two. After instantiating $y$ to three values in each interval, we get two independent sub-MI problems that contain variable $x_1$ and variable $x_2$ respectively. By solving these sub-problems, we obtain three points fitted by each polynomial $p_{-}(y)$ and $p_{+}(y)$ as shown in Figure~\ref{fig: p(y) curve}. Therefore, we can recover both by polynomial interpolation and can obtain $\MI(\theta_2)$ by Equation~\ref{eq: wmi from sum}. 
Figure \ref{fig: cont aos tree} depicts the search space of our algorithm on interval $[0.5, 1]$.

\begin{algorithm}[tp]
  \caption{\textbf{$\mathit{SMI}$}: Search-Based Model Integration}
  \label{alg:wmi}
  \textbf{Input:} $T$: pseudo tree, $\theta$: SMT($\lra$) theory \\
  \textbf{Output:} $p$: MI of theory $\theta$ 
  \begin{algorithmic}[1]
  \If{$T$ is a forest of trees $T^\prime$}
  \State $\theta^\prime \leftarrow$ sub-theories containing variables in $T^\prime$\\
  \Return $\prod_{T^\prime} \mathit{SMI}(T^\prime, \theta^\prime)$
  \EndIf  
  \State $p = 0, y = root(T)$, $ST_y = $ set of subtrees below $y$
  \State $I$ = $\PENODE(\theta, y)$ 
  \ForAll{polynomial piece $\{[l,u], d\} \in I$}  
  \State select $d+1$ distinct values $\alpha_i$'s in $[l, u]$
  \State $p_i \leftarrow  \mathit{SMI}(ST_r, \theta\mid_{(y = \alpha_i)})$
  \State $p_{l,u}(y) \leftarrow$ polynomial interpolation on $(\alpha_i, p_i)$'s  
  \State $p \leftarrow p + \int_l^u p_{l,u}(y) \dif y$
  \EndFor
  \State \textbf{return} $p$
  \end{algorithmic}
\end{algorithm}

The above discussion has shown that for MI problems,
we can instantiate a real variable to finitely many values, decompose the problem into independent parts,
and then solve the sub-problems recursively.
Algorithm~\ref{alg:wmi} follows exactly this strategy for search-based model integration.
The role of pseudo trees will be explained in Section~\ref{sec: complexity}.
Details on caching to speed up the algorithm are included in Appendix~\ref{s:caching}. 
The remaining problem is how to exactly obtain pieces $[l, u]$ and their associated degrees $d$ in function $\PENODE$.
We address this problem next.

\subsection{FINDING PIECES VIA CRITICAL POINTS}
\label{sec: Piece Enum}

Recall that by Proposition~\ref{pro: piece poly}, WMI of SMT($\lra$) theory $\theta$ can be rewritten as $\WMI(\theta, w \mid \vv{x}, \vv{b}) = \int_{I} p(y) \dif y$
where $p(y)$ is a piecewise polynomial,
set $I$ is a union of disjoint support of polynomials in $p(y)$,
and each piece $[l, u] \in I$ is associated with a polynomial degree $d$.
We hope that when a real variable $y$ in theory $\theta$ is chosen to be instantiated,
we can exactly find all pieces and their associated degrees for piecewise polynomial $p(y)$.

It turns out that this can be achieved. While integrating over satisfying assignments with respect to a certain variable given an SMT($\lra$) theory, integration upper bounds and lower bonds are defined by its literals. Changes in integration bounds give rise to different pieces of integration and therefore result in the piecewise nature of the polynomial in Proposition~\ref{pro: piece poly}. In our method we determine these pieces by collecting points where certain bounds meet. Further, by propagating polynomial piece and degree information in a bottom-up manner along the primal graph, we can obtain the pieces and degree for the chosen piecewise polynomial.

We will first describe our method in a basic case where there are only two real variables in the theory. Then we extend this approach to theories with tree primal graphs.

\subsubsection{Base Case: Pieces of Two Real Variables}

First we investigate a simple case where there are only two real variables $x$ and $y$ in SMT($\lra$) theory $\theta$.
Recall that we are solving an unweighted MI problem.
We would like to find pieces and associated degrees for variable $y$ such that we can instantiate $y$ as in Section~\ref{sec: var ini}:
\begin{shrinkeq}{-1ex}
\begin{align*}
  p(y) 
  &= \int_{\theta(x, y)} 1 \, \dif x 
  = \sum_{[l(y), u(y)] \in I(y)}
  \int_{l(y)}^{u(y)} 1 \, \dif x \\
  &= \sum_{[l(y), u(y)] \in I(y)}
  u(y) - l(y)
\end{align*}
\end{shrinkeq}
where set $I(y)$ is defined as 
\begin{shrinkeq}{-1.5ex}
\begin{equation}
\label{def: int bdd set}
  \{ [l(y), u(y)] \mid \forall x \in [l(y), u(y)],\theta(x, y) \text{ is SAT} \}.  
\end{equation}
\end{shrinkeq}
That is, for any fixed value $y^*$, the set $I(y^*)$ consists of intervals of consistent values for variable $x$.
For any $[l(y), u(y)] \in I(y)$, it gives a pair of integration bounds for variable $x$. Further by integrating over $x$ we can obtain a polynomial with respect to variable $y$.

Each piece $[l, u]$ corresponds to a certain class of values that gives the same symbolic integration bounds to variable $x$.
The two values $y = l$ and $y = u$ are endpoints of the piece only if integration bound set $I(y)$ changes at these points,
since the piecewise polynomial $p(y)$ is defined by these bounds.
That is, for arbitrarily small $\epsilon$,
we have $I(l-\epsilon) \neq I(l+\epsilon)$,
and it also holds at point $y = u$.
We formally define critical points below.
\begin{mydef}{\textbf{(Critical Point)}}     
  Let $\theta$ be an SMT($\lra$) theory with two real variables,
  and denote one of the real variables by $y$.
  Let $I(y)$ be an integration bound set as defined in Equation~\ref{def: int bdd set}.
  Then $y = \alpha$ is a critical point
  if for arbitrarily small $\epsilon$,
  it holds that $I(\alpha-\epsilon) \neq I(\alpha+\epsilon)$.  
\end{mydef}
\paragraph{Remark.}
The comparison of set $I(y)$ is done symbolically.
That is, for two distinct values $\alpha, \beta$, we say $I(\alpha) = I(\beta)$ if they have the same set of symbolic integration bounds.
For example, if at $y = \alpha$, $I(y) = \{ [1, y] \}$
and at $y = \beta \neq \alpha$, $I(x) = \{ [1, y] \}$, 
it holds that $I(\alpha) = I(\beta)$.
However, if at $y = \alpha$, $I(y) = \{ [1, y] \}$
and at $y = \beta$, $I(y) = \{ [y, 2] \}$, 
then we say $I(\alpha) \neq I(\beta)$. 

Our idea is that,
if we can find all critical points $y = \alpha$ where the set $I(y)$ changes,
then we can partition real domains of $y$ into disjoint intervals,
such that any support of piecewise polynomial $p(y)$ is either one of these intervals or a union of some intervals.
For the resulting interval $[l, u]$,
we can apply an SMT($\lra$) solver to $\theta^\prime = \theta \land (l < y < u)$ to check whether it is a satisfiable piece of function $p(y)$;
if this is true,
we can obtain the polynomial degree of $p_{l,u}(y)$ defined over this piece by simply traversing theory $\theta^\prime$.
We summarize this procedure as $\PEEDGE$ in Algorithm~\ref{alg: two pe algorithms} in Appendix~\ref{sec: appendix pe}.

\subsubsection{General Case: Pieces of Tree Structures}

Given an SMT($\lra$) theory $\theta$ with a tree-shape primal graph $G$, 
our goal is to enumerate pieces and their associated degrees for the root variable $y$, building on the algorithm we developed in the base case above.
This can be done in a bottom-up manner with tree primal graphs.

Specifically, we first partition theory $\theta$ into sub-theories $\theta_{r, c}$ and $\theta_{G_c}$ for each $c$,
such that $\theta = \bigwedge_{c} (\theta_{r, c} \land \theta_{G_c})$, where variables $c$ are the child variables of root $r$, and graph $G_c$ is the sub-tree rooted at variable $c$.
Each theory $\theta_{r,c}$ contains only variables $r$ and $c$,
on which we can apply the enumeration for the base case above,
and each theory $\theta_{G_c}$ contains only variables in sub-tree $G_c$.
This is possible provided that the primal graph of theory $\theta$ has a tree structure, which is why our algorithm is restricted to SMT($\lra$) theories with tree-shaped primal graphs.

For each child variable $c$,
we first obtain its pieces with respect to theory $\theta_{G_c}$ in a recursive way.
Then we can apply our enumeration algorithm for two-variable theory $\PEEDGE$ to theory $\theta_{r, c}$ with the given pieces of variable $c$.
What we would get are sets of pieces for each child variable $c$.
To be consistent with theory $\theta$,
we need to take intersections of these sets which we refer to as the shattering operation.
Finally, the resulting intersections are pieces and polynomial degrees for root variable $r$.
We provide more details of this procedure called $\PENODE$ in Algorithm~\ref{alg: two pe algorithms} in Appendix~\ref{sec: appendix pe}.

As described above, our piece enumeration algorithm is applicable to MI problems for theories with tree primal graphs.
Moreover, it is also applicable to WMI problems whose SMT theory has a tree primal graph and whose per-literal weights are monomials over univariate literals as described in Section~\ref{sec: MI is all u need}, since our reduction process can preserve the tree structure of the primal graph. 

\subsection{COMPLEXITY ANALYSIS}
\label{sec: complexity}

Inference over networks involving real variables raises considerable challenges for inference, and network structures that are tractable in the discrete case, such as polytrees, give rise to NP-hard inference problems in the hybrid case \citep{koller2009probabilistic}. We show that the complexity of our algorithm is mainly exponentially bounded by the tree height of the primal graph.

Our search algorithm for MI needs to choose which variables to instantiate first. This choice can be based on a tree data structure that orders the variables. Such a tree characterizes the computational complexity as it does for discrete And/Or search algorithms.
We first formally defined the tree that helps guide our search.
\begin{mydef}{\textbf{(Pseudo Tree)}}\label{def:pseudo tree}
  Given an undirected graph $G$ with vertices and edges $(V, E_G)$,
  a pseudo tree for $G$ is a directed rooted tree $T$ with vertices and edges $(V, E_T)$, such that any edge~$e$ that is in $G$ but not in $T$ must connect a vertex in~$T$ to one of its ancestors.
\end{mydef}
That is, edge $e = (v_1, v_2)$ such that  $e \in E_G$ and $e \notin E_T $ implies that either vertex $v_1$ is an ancestor of vertex $v_2$ in $T$ or vertex $v_2$ is an ancestor of vertex $v_1$ in $T$.
Note that the pseudo tree has the same set of vertices as $G$.
Such a pseudo tree guides SMI (Algorithm~\ref{alg:wmi}) in deciding which variable to instantiate, and when to decompose.

Next, we analyze the complexity of SMI.
Since our algorithm performs search,
its time and space complexity is characterized by the size of its search space.
Our analysis does not take caching improvements into consideration.

\begin{figure*}[!tp]
  \begin{subfigure}{0.33\textwidth}
  \centering
    \includegraphics[width=1\linewidth]{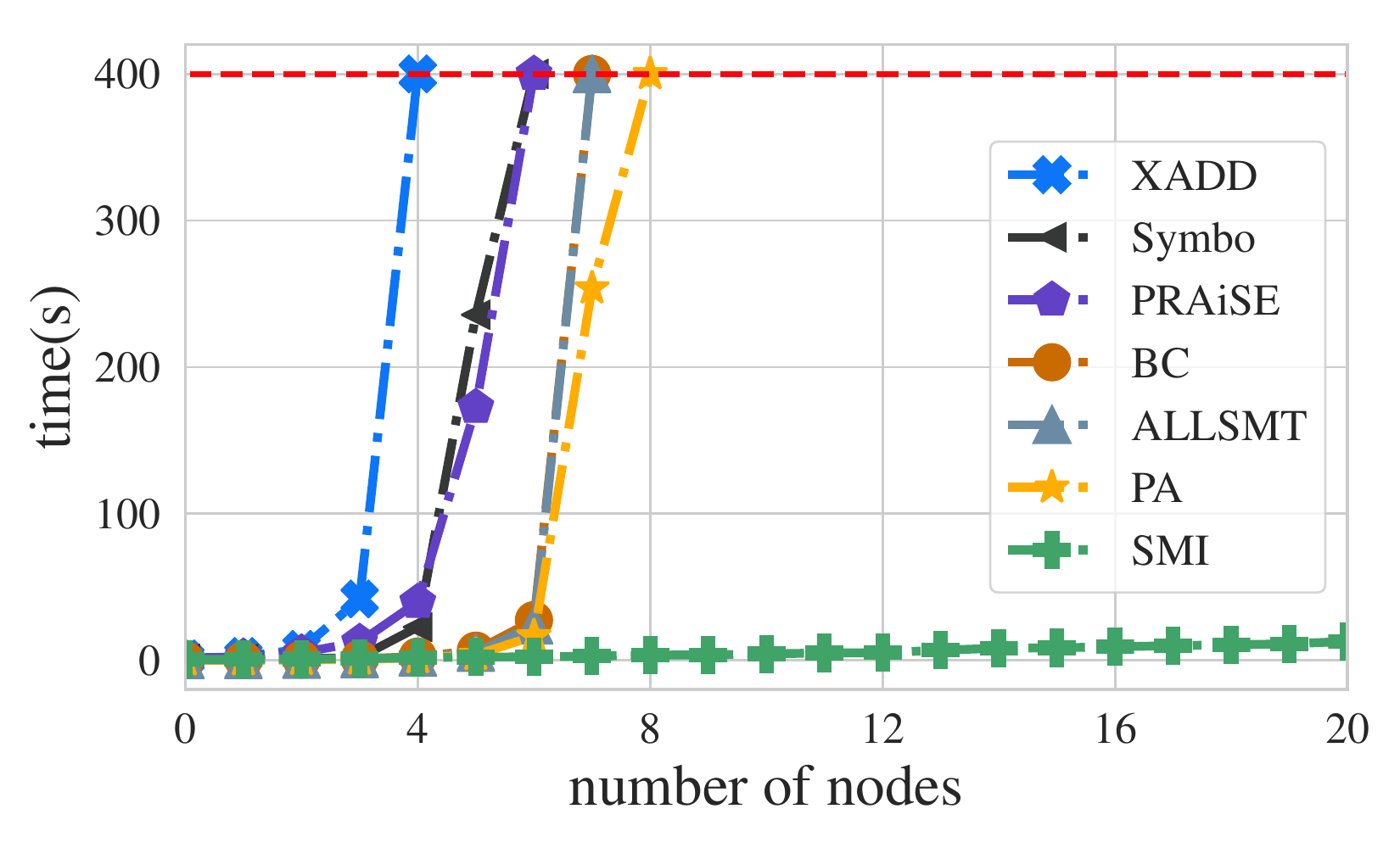}
    \caption{MI runtime on star primal graphs.}
  \end{subfigure}
  ~
  \begin{subfigure}{0.33\textwidth}
  \centering
    \includegraphics[width=1\linewidth]{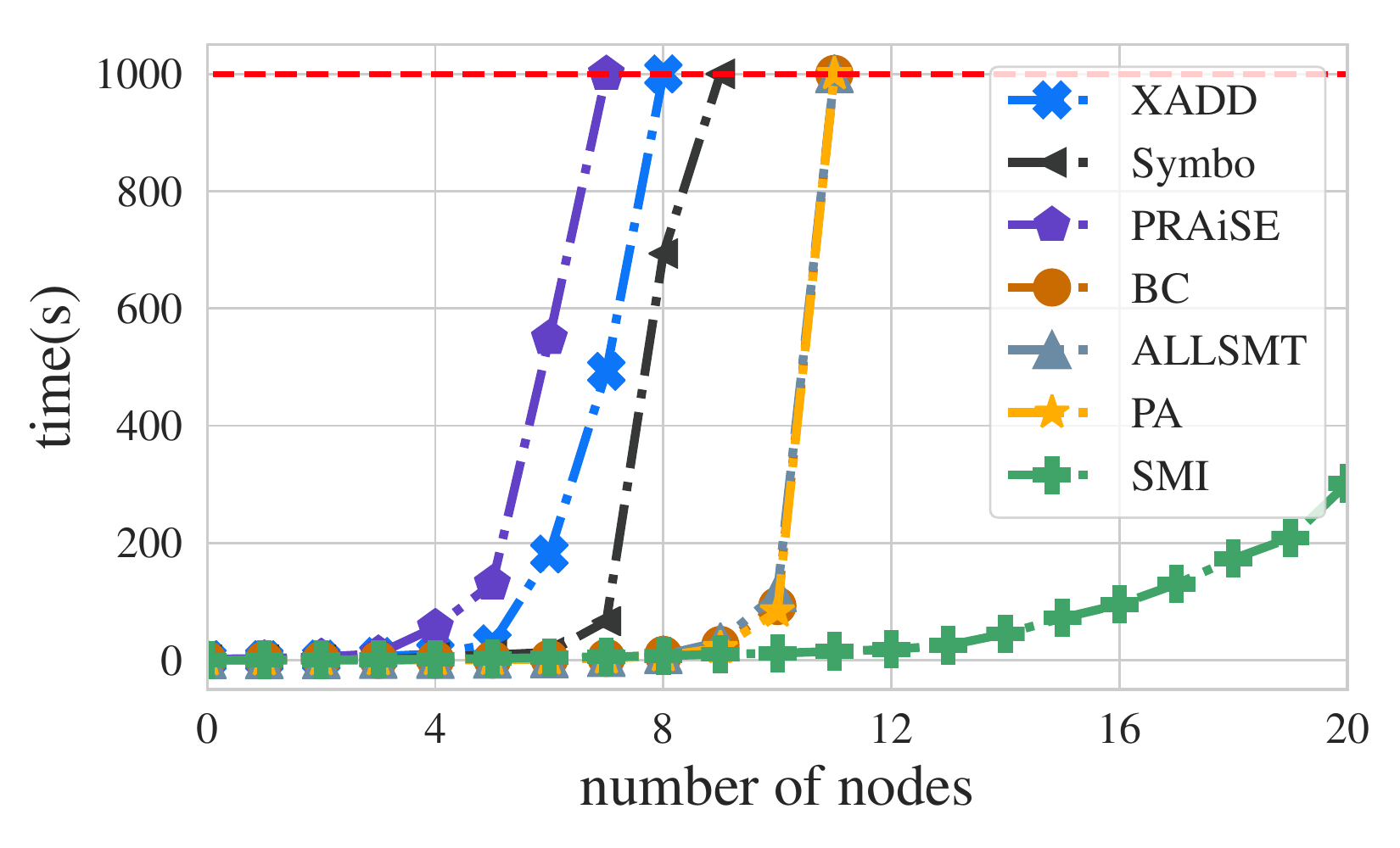}
    \caption{MI runtime on full 3-ary tree graphs.}
  \end{subfigure}
  ~
  \begin{subfigure}{0.33\textwidth}
  \centering
    \includegraphics[width=1\linewidth]{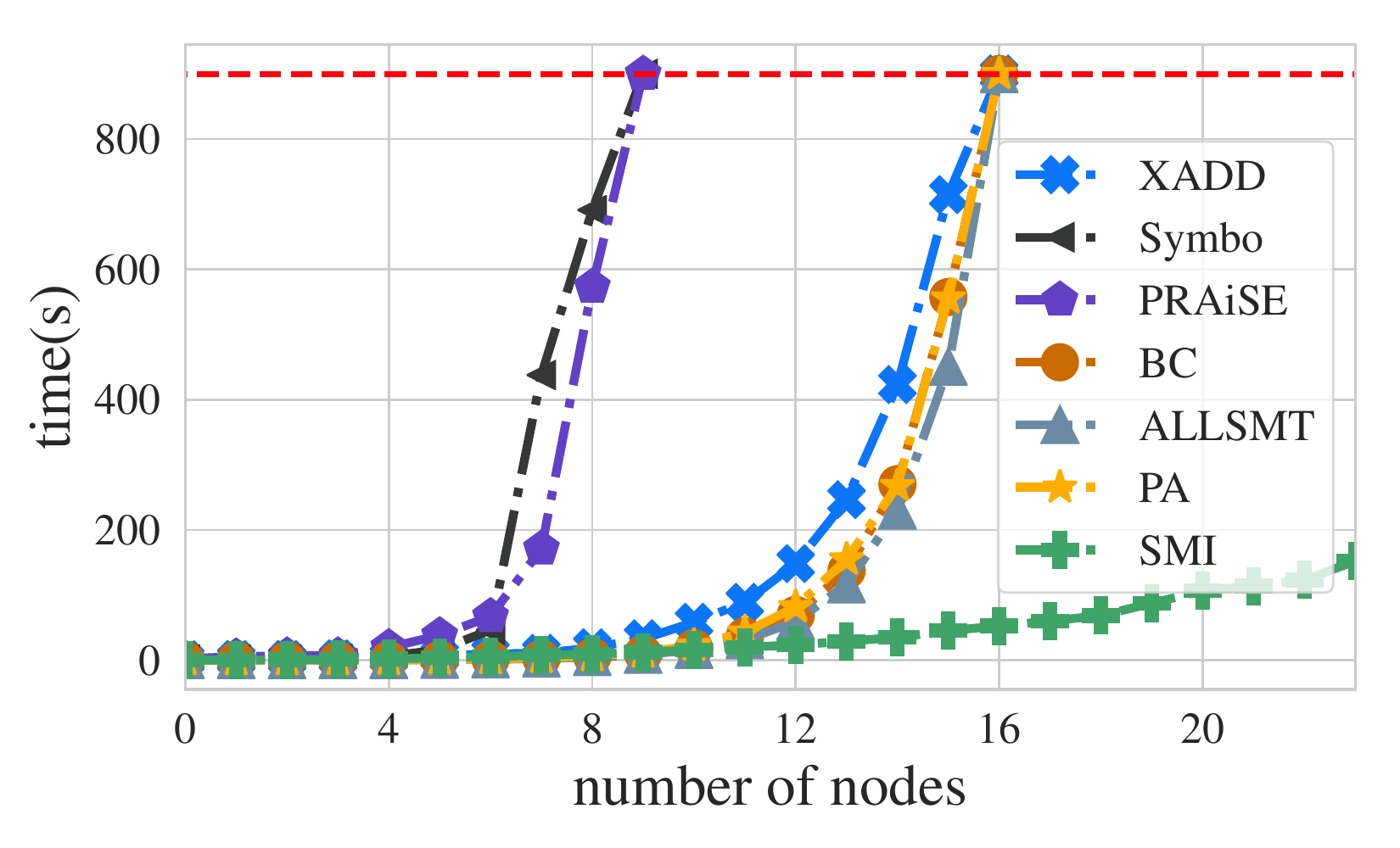}
    \caption{MI runtime on path primal graphs.}
  \end{subfigure}
\\
    \begin{subfigure}{0.33\textwidth}
        \centering
        \includegraphics[width=0.5\linewidth]{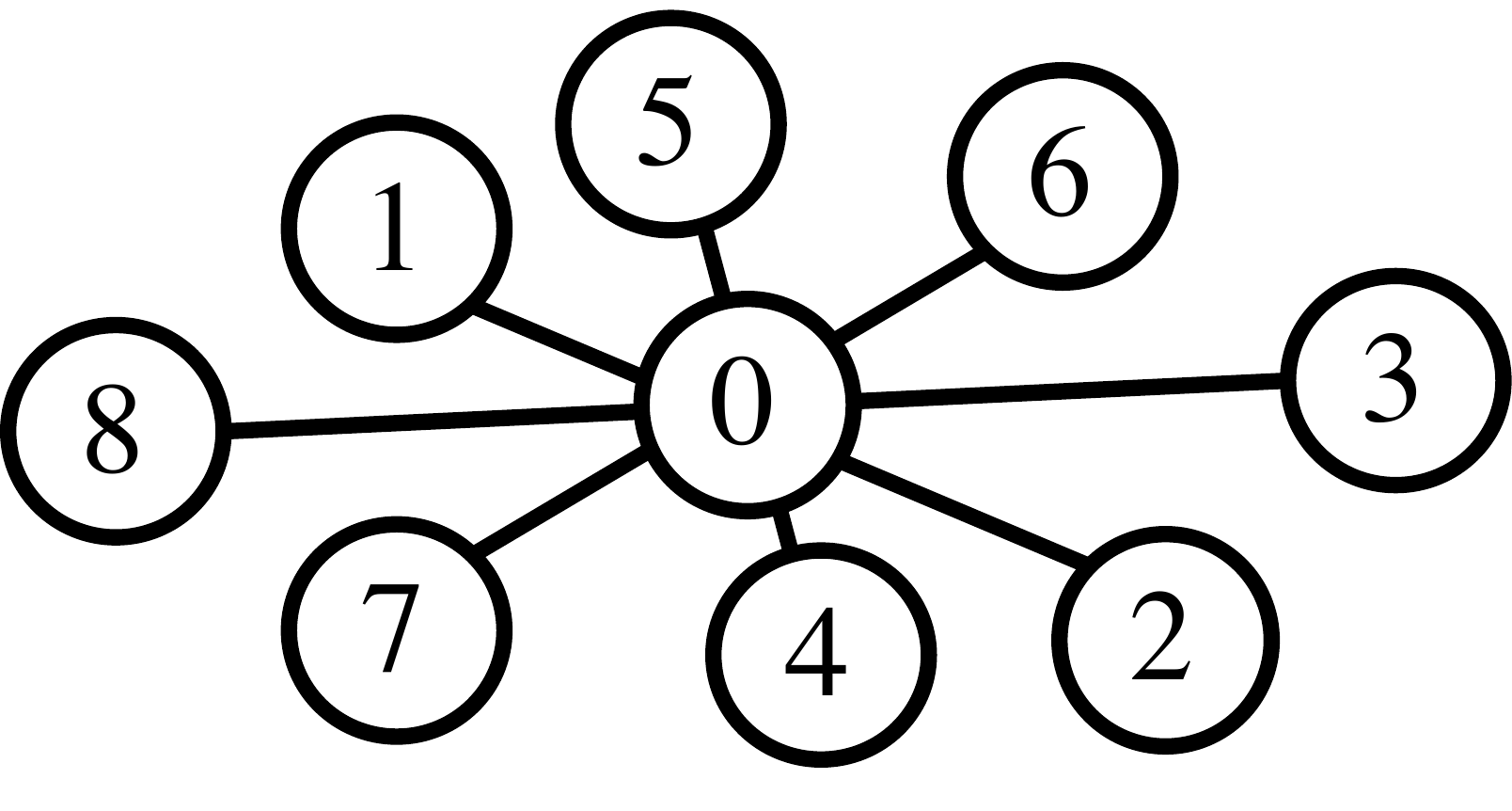}
        \caption{Star primal graph with nodes $n = 8$.}
    \end{subfigure}
  ~
    \begin{subfigure}{0.32\textwidth}
        \centering
        \includegraphics[width=0.5\linewidth]{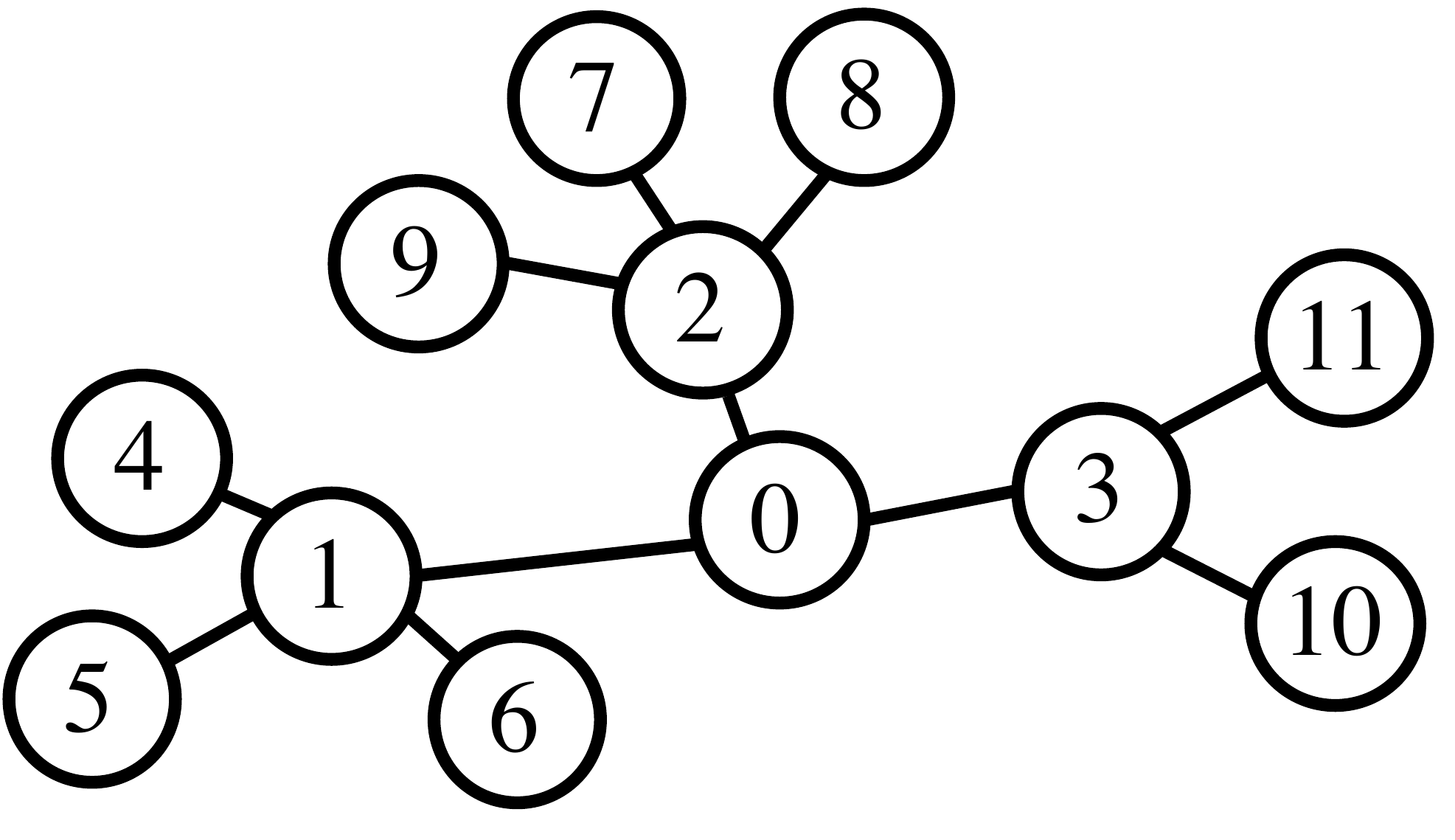}
        \caption{Full 3-ary tree with nodes $n = 12$.}
    \end{subfigure}
  ~
    \begin{subfigure}{0.33\textwidth}
        ~\\~\\
        \centering
        \includegraphics[width=0.45\linewidth]{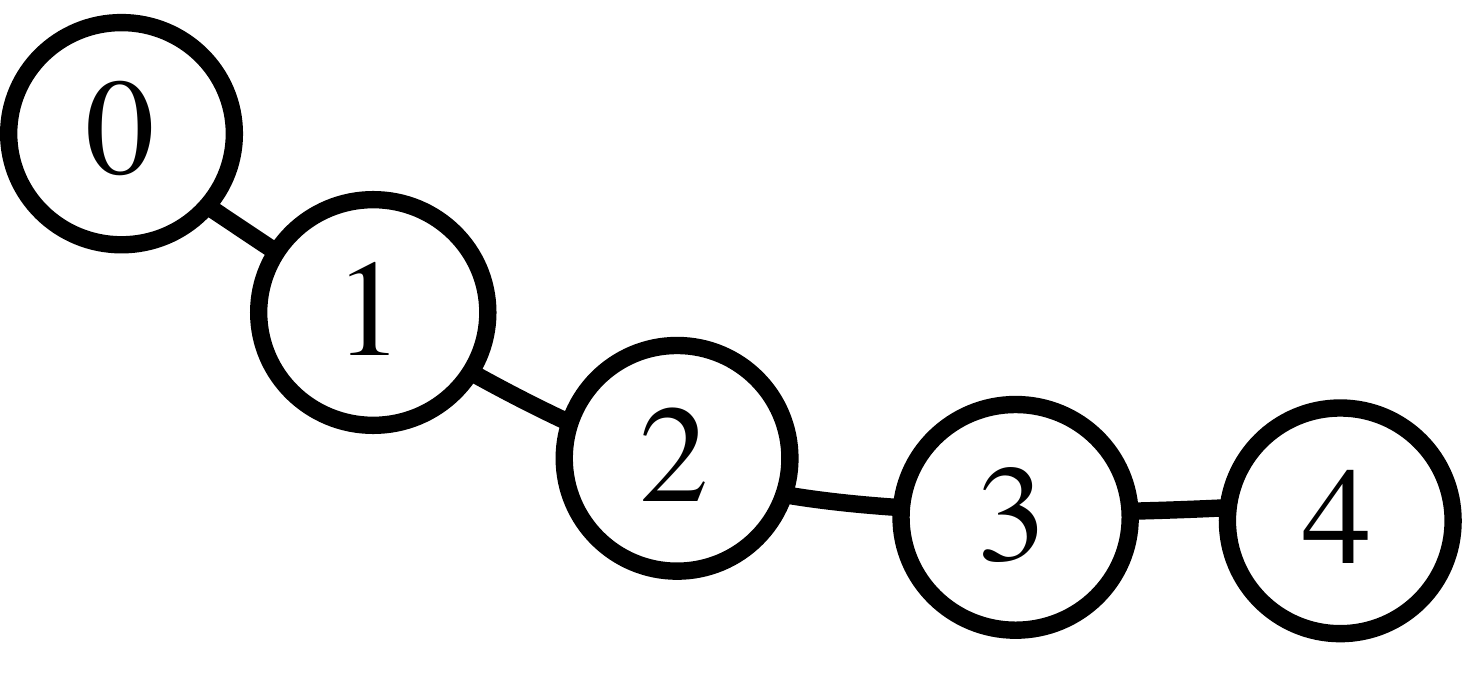}
        \caption{Path primal graph with nodes $n = 5$.}
    \end{subfigure}
    \caption{(a)-(c) MI execution time on SMT($\lra$) with tree primal graphs. (d)-(f) Example tree primal graphs.}
    \label{fig: mi experiment }
\end{figure*}

\begin{thm}{\textbf{(Size of Search Space)}}\label{thm: complexity}
  Consider an SMT($\lra$) theory $\theta$ with a tree-shaped primal graph with height $h_p$,
  and a pseudo tree $T$ with $l$ leaves and height $h_t$.
  Let $m$ be the number of $\lra$ literals in~$\theta$,
  and~$n$ be the number of real variables.   
  Then the size of the SMI search space is   
  $O(l\cdot (n^3 \cdot m^{h_p})^{h_t})$.
\end{thm}

Hence,
we can conclude that the complexity of our algorithm is bounded exponentially by tree heights of both the primal graph and pseudo tree.
In fact, for any tree-shaped primal graph,
we can always choose a pseudo tree whose height $h_t$ is $O(\log n)$ to guide the search \citep{dechter2007and}.
Moreover, the number of leaves $l$ in pseudo tree $T$ is no larger than the number of nodes $n$.
Thus, we have the following corollary.

\begin{cor} \label{cor: complexity}
  Following the notation in Theorem~\ref{thm: complexity},
  with properly chosen pseudo tree $T$ whose tree height $h_t$ is $O(\log n)$,
  the size of the search space generated by SMI is 
  $O\!\left(n^{1 + 3 \log n + h_p \log m}\right)$. 
\end{cor}

Therefore, the complexity of our algorithm is mainly decided by tree heights of primal graphs $h_p$.
In the worst case when tree primal graphs have height $O(n)$ -- for instance path graphs, whose tree height is $n$ when rooted at the start node --
then the worst-case complexity of our algorithm is $O(n^{n \log m})$ by Corollary~\ref{cor: complexity}. That is, the time complexity is worst-case super-exponential.

In cases when the tree primal graph has tree height of size $O(\log n)$,
the complexity of our algorithm is $O(n^{1 + (3 + \log m) \log n})$ which is of quasi-polynomial complexity, and considered to be efficient.
Trees with tree height in $O(\log n)$ are a general class of trees used in various models.
Balancing trees like AVL trees and full k-ary trees are of tree height $O(\log n)$.
Another example is a star graph, which has one internal node and all other nodes as leaves.
This graph corresponds to the well-known naive Bayes structure for directed graphical models.
It is the primal graph of a theory modeling independent variables predicting one and the same dependent (class) variable.
The tree height of star graphs is constant $1$ when choosing the internal node as root.
Hence, our algorithm runs efficiently on such WMI problems.

\section{EMPIRICAL EVALUATION}

We analyze the performance of our search-based MI algorithm on SMT($\lra$) theories with tree primal graphs.
First, we show that our algorithm is efficient for theories whose primal graphs have constant tree heights, or tree heights of log scale w.r.t.\ the number of real variables $n$.
For theories whose primal graph has tree height in $O(n)$ 
-- the cases where our algorithm has super-exponential worst-case complexity in theory --
empirical results show that our algorithm still runs efficiently.
We also consider a more complex house price model where house sizes are dependent, as opposed to those in Example~\ref{exa: indept house}. Moreover, the house price model has non-trivial weight functions that our algorithm first reduces to a MI problem as outlined in Section~\ref{sec: MI is all u need}.
We compare our algorithm to alternative WMI solvers and conclude that it significantly outperforms existing solvers on these benchmarks.

\paragraph{Benchmarks}
We compare our algorithm (SMI) with other WMI solvers.
The block-clause-strategy-based solver (BC) \citep{belle2015probabilistic} iteratively generates new models by adding the negation of the latest model to the formula for the following iteration.
The all-satisfying-assignments-based solver (ALLSMT) \citep{belle2016component} first generates the set of all $\lra$-satisfiable total truth assignments on atoms that propositionally satisfy the theory.
The implementation of \citet{braz2016probabilistic} (PRAiSE) is a variable-elimination-based solver.
The predicate-abstraction-based solver (PA) \citep{morettin2017efficient} exploits the power of SMT-based predicate abstraction to reduce the number of models to be integrated over.
Both the extended algebraic-decision-diagram-based solver (XADD) \citep{kolb2018efficient} and sentential-decision-diagram-based solver (Symbo) \citep{zuidberg2019exact} use circuit-based compilation languages and exploits the circuit structures.

\subsection{TREE PRIMAL GRAPHS}
We investigate the performance of our algorithm on SMT($\lra$) theories with three types of tree primal graphs:
1) star graphs, consisting of one center node connected to all other nodes, and no other connections;
2) full 3-ary trees, whose non-leaf vertices have exactly three children and all levels are full except for some rightmost position of the bottom level;
3) path graphs, consisting of linearly connected nodes.
These structural constraints arise naturally in data and many probabilistic graphical modeling problems.

For each graph type, given a number of nodes $n$, we introduce $n$ real variables $\vv{x} = \{x_0, x_1, \cdots, x_{n-1}\}$ with bounded domains $\forall i, (-1 \leq x_i \leq 1)$. 
Denote the graph by $G = (V, E)$ where $V = \{ 0, 1, \cdots, n - 1 \}$ is the vertex set and $E = \{ (i, j), i, j \in V \}$ the edge set.
We perform MI for the following theories and increasing $n$.
\begin{equation*}
  \theta(\vv{x}) =
  \left\{
  \begin{array}{l}
    \bigwedge_{i \in V}~ (-1 \leq x_i \leq 1) \\
    \bigwedge_{(i, j) \in E}~((x_i + 1 \leq x_j) \lor (x_j \leq x_i - 1 )) \\
  \end{array}
  \right.
\end{equation*}

 Figure~\ref{fig: mi experiment } shows example primal graphs and the execution time of experiments comparing SMI with baselines.
 
For MI over theories with all three types of tree primal graphs,
our algorithm significantly outperform other WMI solvers in terms of execution time.
The runtime curves of other solvers grow seemingly exponentially while our curve grows slowly with the number of real variables.
For theories with star graphs and full 3-ary trees as primal graphs,
the time curves of SMI are consistent with the complexity analysis in Section~\ref{sec: complexity} stating that our algorithm has quasi-polynomial complexity.
For theories with path graphs as primal graphs, which are still sparse graphs,
we perform caching and
the runtime curve grows slowly, even though our worst-case analysis allows for a super-exponential time complexity.

\begin{figure}[!t]
  \begin{subfigure}[t]{0.35\textwidth}
  \centering
    \includegraphics[height=1.3in]{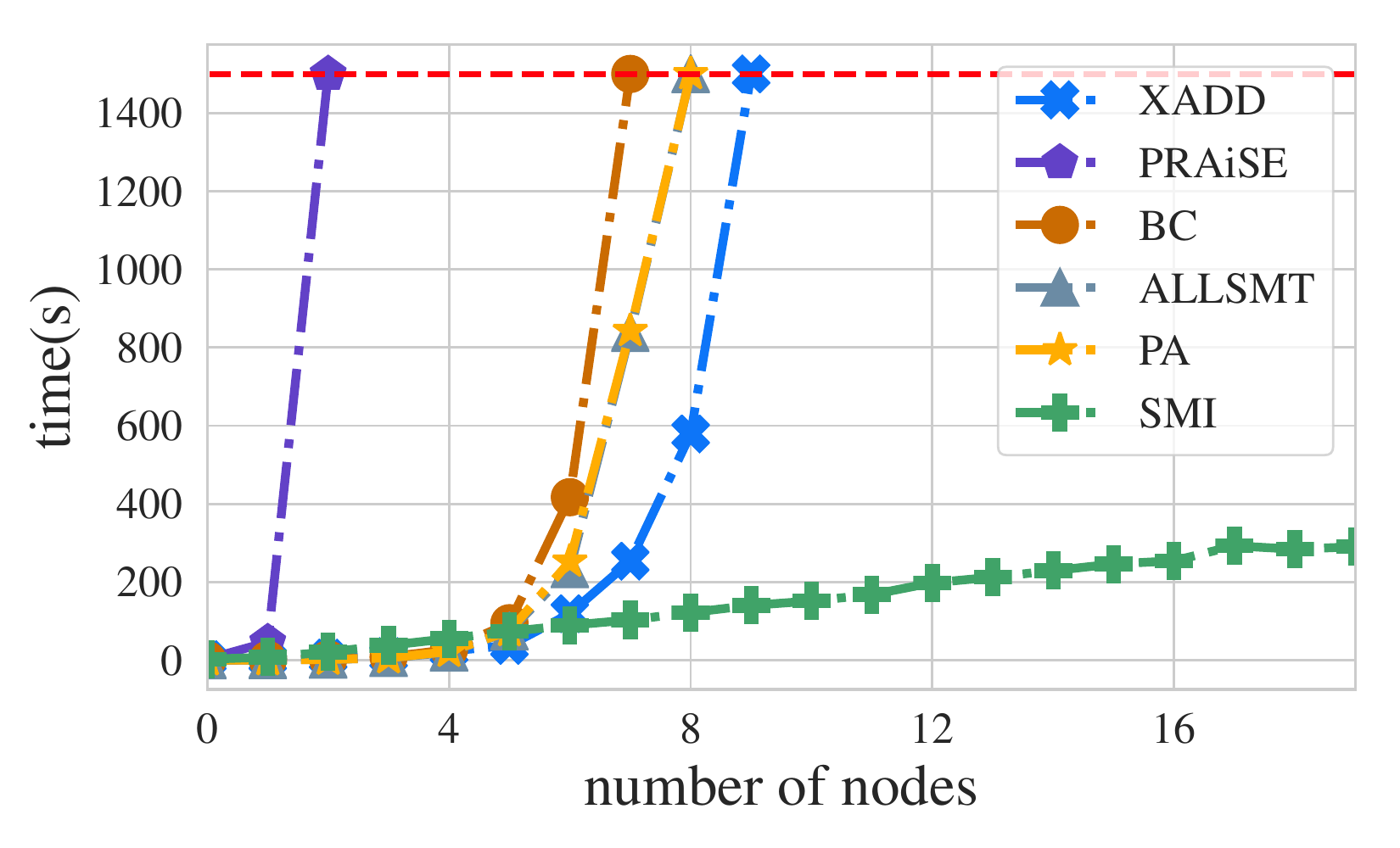}
  \end{subfigure}
  \begin{subfigure}[t]{0.1\textwidth}
    \includegraphics[height=1.3in]{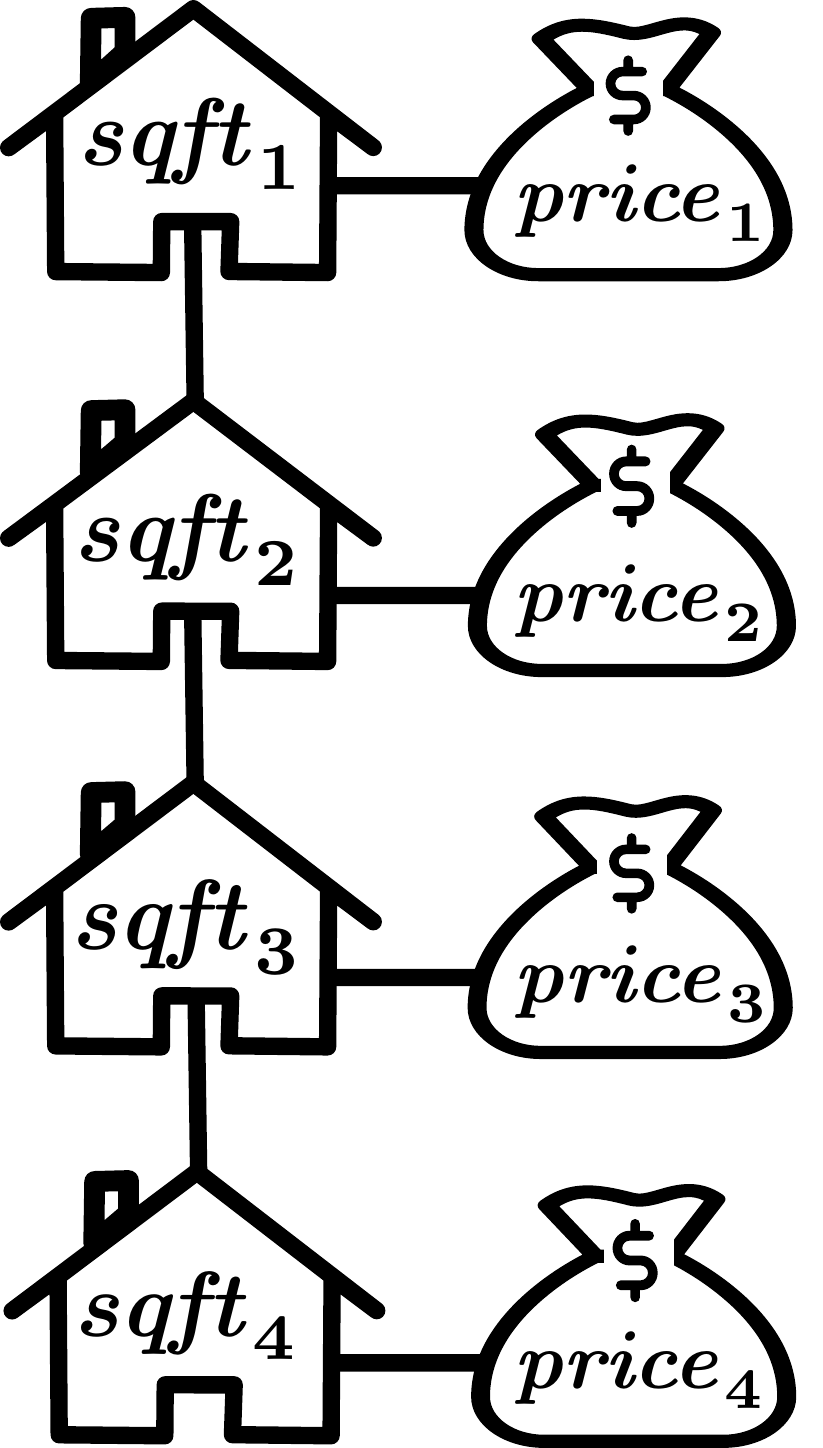}
  \end{subfigure}
  \caption{Runtime and primal graph for house price model.}
  \label{fig: house price experiment}
\end{figure}

\subsection{HOUSE PRICE SMT($\lra$) MODEL}
In Example \ref{exa: indept house} we performed MI for multiple houses based on extreme independence assumptions. Now we consider a more complicated case where houses are not independent and there are Boolean variables in the SMT($\lra$) model. Moreover, we choose non-trivial per-literal weight functions in order to evaluate our algorithm for reducing WMI to unweighted MI problems.

Specifically, we consider $n$ houses that are located along a street. Each house $i$ has its price and square footage model as in Example~\ref{exa: price}. Also, we enforce the constraint that square footage between two neighboring houses should not vary too much and we use a Boolean variable $b$ to indicate whether or not these houses are located in an urban area. 
This gives the following SMT theory.
\begin{equation*}
    \gamma_{\mathit{street}} = 
    \left\{
    \begin{array}{l}
        (b \lor \neg b) \land \bigwedge_{i=1}^n \gamma_i \\
        \bigwedge_{i=1}^{n-1} (\squarefeet_{i} \le \squarefeet_{i+1} + \offset)
    \end{array}
    \right.
\end{equation*}
with $\offset$ a constant characterizing maximum difference in square footage between two neighboring houses.
For weights $w$,
consider the set of literals $\mathL = \{b\} \cup \{0 \! < \! \price_i \! < \! 3000, i = 1, \cdots, n\}$ and per-literal weight functions $\mathP = \{p_{b}\} \cup \{p_{(0 < \price_i < 3000)}, i = 1, \cdots, n \}$, with $p_b(\vv{x}) = 1.5$ and $p_{(0 < \price_i < 3000)}(\vv{x}) = \price_i^2$ for all $i$.
Then, in worlds where all literals in $\mathL$ are satisfied, our weight function is $1.5 \prod_{i=1}^n \price_i^2$.
In worlds where $b$ is false but other literals are satisfied, the weight function is $\prod_{i=1}^n \price_i^2$.
Figure~\ref{fig: house price experiment} shows an example primal graph and WMI runtime for this house price model.

\section{RELATED WORK}
SMT \citep{barrett2010smt} has been one of the most prominent advances in automated reasoning
and many efficient SMT solvers have been built \citep{de2008z3,barrett2011cvc4,cimatti2013mathsat5,dutertre2014yices}.
The counting version of SMT, that is \#SMT, and in particular \#SMT($\la$) is a fundamental problem in quantitative program analysis
\citep{liu2011program,geldenhuys2012probabilistic,filieri2014reliability,phan2014quantifying,von2015symbolic,filieri2015quantification}.
The \#SMT($\la$) problem is \#P-hard, as is model counting  \citep{valiant1979complexity}.
Other first-order hybrid probability models have been proposed, usually based on sampling inference~\citep{ravkic}.

SGDPLL($T$) is an algorithm for solving probabilistic inference modulo theories while also generating simpler sub-problems \citep{braz2016probabilistic}.
It performs case analysis on SMT literals, whereas SMI instead operates on continuous theory variables.
Similar to our observation that WMI problems can be reduced to MI problems, \citet{chakraborty2015weighted} propose a method to reduce WMC to unweighted model counting.
Although the focus of this paper is on exact inference, there also exist notable approximate solutions to \#SMT($\la$) and WMI
\citep{ma2009volume,belle2015hashing,chakraborty2016approximate,chistikov2017approximate}.

\citet{morettin2017efficient}  enumerate integrable spaces by predicate abstraction and allow general weight functions. \citet{kolb2018efficient} use case functions as weights, which still permits compilation into XADD circuits. 
Weight functions in these two cases are not consistent with the factorization structure of the SMT sentence. The factorization structure is a crucial aspect of efficient inference, and its isolation to the logical part of WMC/WMI is considered to be an advantage, facilitating solver building. 
Our definition of factorized weight functions is similar to \citet{belle2015probabilistic} and \citet{zuidberg2019exact}. \citet{belle2016component} exploit independence in WMI problems that are exactly equivalent to WMC problems.

\section{Conclusions}
This paper proposed a search-based WMI algorithm that exploits structural independence properties to improve efficiency.
For WMI on SMT($\lra$) theories with tree primal graphs and piecewise polynomial weight functions,
our algorithm decomposes WMI problems during search.
A complexity analysis showed that 
for balanced tree primal graphs, our algorithm yields quasi-polynomial complexity.
Experimental comparisons confirmed a drastic efficiency improvement over baselines.

\guy{TODO: mention related work by Shenoy on continuous-variable junction trees? \citep{shenoy2011inference}}

\guy{TODO: mention generalized DPLL?}

\guy{Should we mention Braz et al. \citep{braz2016probabilistic} and how it compares? I think they focus on discrete theories of integers, so that allows them to decompose and use sparsity. They don't support LRA.}

\guy{In around Sept-Oct I sent an unpublished paper with good related work references on \#SMT from the verification community. We should make sure to include all of those as well.}

\guy{Kuldeep Meel had a paper \citep{chakraborty2015weighted} at IJCAI 2015-2016 or so about reducing Weighted Model Counting to Model Counting. We should cite it as related to our WMI to MI reduction.}

\zhe{weight functions we use is different from \citet{morettin2017efficient}, who do integration over total assignments over integratable space thus general function, and from \citet{kolb2018efficient} who use knowledge compilation thus case function.}

\zhe{Double check acknowledgement.}

\zhe{Double check UAI camera ready instruction.}

\paragraph{Acknowledgements}
The authors  would  like  to  thank  Brendan  Juba, Andrea Passerini, and Roberto Sebastiani for  valuable discussions.
This work is partially supported by NSF grants \#IIS-1657613, \#IIS-1633857, \#CCF-1837129, DARPA
XAI grant \#N66001-17-2-4032, NEC Research, and gifts from Intel and Facebook Research.

\bibliographystyle{plainnat}
\bibliography{uai19}

\clearpage

\appendix

\section{PROOFS}

\subsection{PROOF OF PROPOSITION~\ref{pro: b to nb wmi}}

\begin{proof}{\textbf{(Proof of Proposition \ref{pro: b to nb wmi})}}
  
  Consider the most basic case when there is only one Boolean variable $b$ in theory $\theta$.
  Let $\theta^\prime$ be an SMT($\lra$) theory defined as follow
  \begin{align*}
    \theta^\prime &= \theta\{ b: \lambda_b\} \land (-1 \le \lambda_b \le 1)
  \end{align*}
  where $\theta\{b: \lambda_b\}$ is obtained by replacing all atom $b$ by $0 < \lambda_b$ and replacing all its negation $\neg b$ by $\lambda_b < 0$ in theory $\theta$.
  
  Recall that weight functions are defined by a set of literals $\mathL$ and a set of per-literal weight functions $\mathP = \{ p_{\ell}(\vv{x}) \}_{\ell \in \mathL}$. When a literal $\ell$ is satisfied in a world, denoted by $\vv{x} \land \vv{b} \models \ell$, weights are defined as follows
    \begin{equation*}
        w(\vv{x}, \vv{b}) = \prod_{\substack{\ell \in \mathL \\ \vv{x} \land \vv{b} \models \ell}} p_{\ell}(\vv{x}). 
    \end{equation*}
  
  Let $\mathL^\prime$ be a set of literals obtained by replacing Boolean literal $b$ by $0 < \lambda_b$ and replacing its negation $\neg b$ by $\lambda_b < 0$ in theory $\theta$ as we do for theory.
  For the set of per-literal weight functions $\mathP^\prime$, we define it for introduced real variable $\lambda_b$ by $p_{(\lambda_b > 0)} = p_b$ and $p_{(\lambda_b < 0)} = p_{\neg b}$.
  
  Then we have that for any $\vv{x}^*$,
  \begin{align*}
      w^\prime(\vv{x}^*, \lambda_b) = 
      \left\{
      \begin{array}{lr}
        w(\vv{x}^*, b), & 1 > \lambda_b > 0 \\
        w(\vv{x}^*, \neg b), & -1 < \lambda_b < 0 \\
      \end{array}
      \right.
  \end{align*}

  By definition of WMI, we write $\WMI(\theta, w \mid \vv{x}, \vv{b})$ in its integration form as follows.
  \begin{align*}
    &\WMI(\theta, w \mid \vv{x}, \vv{b}) \\
    =& \int\limits_{\theta(\vv{x}, b)} w(\vv{x}, b) d \vv{x} + \int\limits_{\theta(\vv{x}, \neg b)} w(\vv{x}, \neg b) d \vv{x}
  \end{align*}
  For the first term in the above equation, we can rewrite it s.t. Boolean variable $b$ is replaced by real variable $\lambda_b$ in the following way.
  \begin{align*}
    \int\limits_{\theta(\vv{x}, b)} w(\vv{x}, b) d \vv{x}
  =& \int_0^1\int\limits_{\theta(\vv{x}, b)} w(\vv{x}, b) d \vv{x} d \lambda_b \\
  =& \int\limits_{\theta^\prime(\vv{x}, \lambda_b)} w^\prime(\vv{x}, \lambda_b) d \vv{x} d \lambda_b 
  \end{align*}
  By doing this to the other integration term of $\WMI(\theta, w \mid \vv{x}, \vv{b})$, and also by the definition of WMI, we finally obtain that
  \begin{align*}
      \WMI(\theta, w \mid \vv{x}, \vv{b}) = \WMI(\theta^\prime, w^\prime \mid \vv{x}^\prime)
  \end{align*}
  where $\vv{x}^\prime = \vv{x} \cup \{ \lambda_b\}$ is a set of real variables.
  The proof above can be easily adapted to multiple Boolean variable cases, which proves our proposition.
\end{proof}

\subsection{PROOF OF PROPOSITION~\ref{pro: wmi to mi}}
\label{proof: wmi to mi}

\begin{proof}{\textbf{(Proof of Proposition \ref{pro: wmi to mi})}}
  To start with, we consider SMT($\lra$) theory $\theta$ with no Boolean variables with a simple weight function $w$
  where the set of literal $\mathL = \{ \ell \}$ has only one literal and literal weight function $p_\ell(\vv{x}) = \prod_{i=0}^n x_i^{p_i}$.

  \begin{cla}\label{cla: mono and mi}
    For a monomial function $f(\vv{x}) = \prod_{i=0}^n x_i^{p_i}$,
    let $\theta_f = \bigwedge_{i = 0}^n \bigwedge_{j = 1}^{p_i} (0 \leq z_j^i \leq x_i)$.
    Then we have the monomial $f(\vv{x}) = \textit{MI}(\theta_f \mid \vv{z}; \vv{x})$,
    where $\vv{z}$ is the set of real variables $z_j^i$ in theory $\theta_f$, 
    and $\vv{x}$ is parameters of theory $\theta_f$.
  \end{cla}

  Let $\theta^\prime = \theta \land (\ell \Rightarrow \theta_p) \land (\neg \ell \Rightarrow \hat{\theta}_p)$ where $p = p_{\ell}$ for brevity, $\theta_p$ is as defined in Claim \ref{cla: mono and mi} and $\hat{\theta}_p := \bigwedge_{i = 0}^n \bigwedge_{j = 1}^{p_i} (0 \leq z_j^i \leq 1)$.
  Then we can rewrite $\WMI(\theta, w \mid \vv{x})$ as MI problem by Claim \ref{cla: mono and mi} as follows.
  \begin{align*}
    &\WMI(\theta, w \mid \vv{x})
    = \int_{\theta(\vv{x})} w(\vv{x}) d \vv{x} \\
    & = \int_{\theta(\vv{x}) \land \ell(\vv{x})} p(\vv{x}) d \vv{x} +  \int_{\theta(\vv{x}) \land \neg \ell(\vv{x})} 1 d \vv{x}\\
    & = \int_{\theta(\vv{x}) \land \ell(\vv{x})} MI(\theta_p \mid \vv{z} ; \vv{x}) d \vv{x} +  \int_{\theta(\vv{x}) \land \neg \ell(\vv{x})} 1 d \vv{x}\\
    & = \int_{\theta(\vv{x}) \land \ell(\vv{x})} \int_{\theta_p(\vv{z})} 1 d \vv{z} d \vv{x} +  \int_{\theta(\vv{x}) \land \neg \ell(\vv{x}) \land \hat{\theta}_p} 1 d \vv{x} d \vv{z}\\
    & = \MI(\theta \land (\ell \Rightarrow \theta_p) \land (\neg \ell \Rightarrow \hat{\theta}_p) \mid \vv{x}, \vv{z})
  \end{align*}
  Take $\vv{x}^\prime = \vv{x} \cup \vv{z}$ then the proposition holds.
  The proof can be easily adapted for monomials with non-trivial coefficient by inducing more real variables $z$.
  It also holds for more general weight functions with literal set $\mathL = \{ \ell_i \}_{i = 1}^k$ and set of monomial per-literal weight functions $\mathP = \{ p_{\ell_i} \}_{i = 1}^k$,
  by taking theory $\theta^\prime$ as follows which completes the proof of proposition.
  \begin{align*}
    \theta^\prime = \theta \land \bigwedge_{i = 1}^k (\ell_i \Rightarrow \theta_{p_{\ell_i}}) \land \bigwedge_{i = 1}^k (\neg \ell_i \Rightarrow \hat{\theta}_{p_{\ell_i}}).
  \end{align*}

\end{proof}

\begin{proof}{\textbf{(Proof of Claim \ref{cla: mono and mi})}}
  By definition of theory $\theta_f$,
  \begin{align*}
    MI(\theta_f & \mid \vv{z}; \vv{x}) 
    = \int_{\theta_f (\vv{z})} 1 d \vv{z} \\
    & = \prod_{i = 1}^n \prod_{j = 1}^{p_i} \int_0^{x_i} 1 d z_j^i \\
    & = \prod_{i = 1}^n \prod_{j = 1}^{p_i} x_i = \prod_{i = 1}^n x_i^{p_i} = f(\vv{x}).
  \end{align*}
\end{proof}

\subsection{REDUCTION TO MI WITH POLYNOMIAL WEIGHTS}
\label{sec: MI with poly weight}

The reduction from WMI problems to MI problems in Proposition~\ref{pro: wmi to mi} can also be done for arbitrary polynomial weight functions but can increase treewidth of primal graphs. We give a formal description on this reduction as follows.

Let $\theta$ be an SMT($\lra$) theory with no Boolean variables with weight functions
where the set of literal $\mathL = \{ \ell \}$ has only one literal and literal weight function is a polynomial, denoted by $p(\vv{x}) = \sum_{i=1}^k \alpha_i f_i(\vv{x})$ with each $f_i$ a monomial function.

It has been shown in the proof of Proposition~\ref{pro: wmi to mi} in Section~\ref{proof: wmi to mi} that for each monomial function $f_i$, there exist two SMT($\lra$) theories $\theta_i$ and $\hat{\theta}_i$ such that $\mi(\theta_i \mid \vv{z}_i; \vv{x}) = f_i(\vv{x})$ and $\mi(\hat{\theta}_i \mid \vv{z}_i; \vv{x}) = 1$.

Let's define theories $\theta_i^\prime = \theta_i \land (0 < v_i < \alpha_i)$ and $\hat{\theta}_i^\prime = \hat{\theta}_i \land (0 < v_i < 1)$ with parameter variables $v_i$. Also define an indicator variable $\lambda$ with real domain $[0, k]$ and literals $\ell_i = i - 1 < \lambda < i$ with $i \in \{1, 2, \cdots, k\}$. 
Then we have that for an SMT($\lra$) theory $\theta^\prime$ defined as follows, it holds that $\WMI(\theta, w \mid \vv{x}) = \mi(\theta^\prime \mid \vv{x}, \vv{z})$ with $\vv{z}$ denoting all auxiliary variables.
\begin{align*}
    \theta^\prime = \theta \land (\ell \iff \lor_{i=1}^k \ell_i) \bigwedge_{i=1}^k (\ell_i \Rightarrow \theta_i^\prime) \bigwedge_{i=1}^k (\neg \ell_i \Rightarrow \hat{\theta}^\prime_i)
\end{align*}
Why the WMI problem and the MI problem are equal can be proved by the following observations.
\begin{align}
\WMI(&\theta, w \mid \vv{x})
= \int_{\theta(\vv{x})} w(\vv{x}) d \vv{x} \\
&= \int_{\theta(\vv{x}) \land \ell(\vv{x})} p(\vv{x}) d \vv{x} +  \int_{\theta(\vv{x}) \land \neg \ell(\vv{x})} 1 d \vv{x}
\label{eq: reduction wmi}
\end{align}

For the first term in Equation~\ref{eq: reduction wmi}, we have that 
\begin{align*}
    \int_{\theta(\vv{x}) \land \ell(\vv{x})} & p(\vv{x}) d \vv{x} = \sum_{i=1}^k \int_{\theta(\vv{x}) \land \ell(\vv{x})} \alpha_i f_i(\vv{x}) d \vv{x} \\
    &= \sum_{i=1}^k \int_{\theta(\vv{x}) \land \ell(\vv{x}) \land \ell_i} \alpha_i f_i(\vv{x}) d \vv{x} d \lambda \\
    &= \sum_{i=1}^k \int_{\theta(\vv{x}) \land \ell(\vv{x}) \land \ell_i \land \theta_i} 1  d \vv{x} d \vv{z} \\
    &= \mi(\theta^\prime \land \ell \mid \vv{x}, \vv{z})
\end{align*}
Also for the second term in Equation~\ref{eq: reduction wmi}, 
it equates to $\mi(\theta^\prime \land \neg \ell \mid \vv{x}, \vv{z})$. Therefore, reduction from the WMI problem to the MI problem holds. Although the reduction process we show here is for theories with one polynomial weight function, this process can be generalized to theories with multiple polynomial weight functions with little modification.

\subsection{PROOF OF PROPOSITION~\ref{pro: piece poly}}
\begin{proof}{\textbf{(Proof of Proposition \ref{pro: piece poly})}}
  It follows from definition of WMI.
  Denote the set of real variables $\vv{x} \backslash \{ y \}$ by $\vv{\hat{x}}$.
  From the definition of WMI in Equation~\ref{def: wmi},
  we can obtain the following partial derivative of WMI of theory $\theta$ w.r.t. variable $y$.
  \begin{align*}
    & \frac{\partial}{\partial x} \WMI(\theta, w \mid \vv{x}, b) \mid_{y = y^*} \\
    &= \sum_{\mu \in \B^m} 
    \int\limits_{\theta(y^*, \vv{\hat{x}}, \mu)} 
    w(y^*, \vv{\hat{x}}, \mu) d \vv{\hat{x}} 
  \end{align*}
  where the variable $y$ is fixed to value $y^*$ in weight function, $\mu$ are total truth assignments to Boolean variables as defined before.
  The weight function is integrated over set 
  $\{ \vv{\hat{x}}^* \mid \theta(y^*, \vv{\hat{x}}^*, \mu ) \text{ is \true} \}$.
  We define $p(y)$ as follows 
  \begin{align*}
    p(y) := \sum_{\mu \in \B^m} 
    \int\limits_{y, \theta(\vv{\hat{x}}, \mu)} 
    w(y, \vv{\hat{x}}, \mu) d \vv{\hat{x}}  
  \end{align*}
  Since weight functions $w$ are piecewise polynomial,
  function $p(y)$ is a univariate piecewise polynomial $p(y)$, and $\WMI(\theta, w \mid \vv{x}, b)$ is an integration over $p(y)$, which finishes our proof.
\end{proof}

\begin{algorithm*}[t]  
  \caption{Polynomial pieces and degree enumeration algorithms}
  \vspace{-0.15in}
  \begin{multicols}{2}      
  {\textbf{a) $\PEEDGE$} -- For Two Variable Theory} \\
  \textbf{Input}:    
  $\theta$: SMT($\lra$) theory with two real variables \\
  $I$ : interval and degree tuples of variable $x$\\
  \textbf{Output}:  
  $I_y$: pieces and degrees for variable $y$
  \begin{algorithmic}[1]
  \State $B \leftarrow$ collect integration bounds on variable $x$
  \State $Y \leftarrow$ $y$ values where two bounds in $B$ meet
  \ForAll{interval $[l, u]$ resulting from $Y$}
  \State $\theta^\prime \leftarrow \theta \land (l \leq y \leq u)$
  \If{ $\theta^\prime$ is SAT}
  \State $\{l(y),u(y),d\} \leftarrow$ get\_bound\_degree($x$,$\theta^\prime$,$I$) 
  \State $d^\prime \leftarrow \text{argmax}_d ~~ \text{get\_degree}(l(y), u(y), d)\})$
  \State $I_y \leftarrow I_y \cup ([l, u], d^\prime)$
  \EndIf
  \EndFor
  \State \textbf{Return} $I_y$
  \end{algorithmic}  
  \columnbreak
  {\textbf{b) $\PENODE$} -- For Tree Primal Graph} \\
  \textbf{Input}:
    $\theta$: SMT theory with tree primal graph\\
    $G$: primal graph for theory $\theta$ \\
    \textbf{Output}:  
    $I_y$: interval and degree tuples of root variable $y$
    \begin{algorithmic}[1]
    \If{root $y$ has no child}
    \State $I_y \leftarrow$ get\_bound\_degree($\theta$) \\
    \Return{$I_y$}
    \EndIf
    \State $\theta_{y,c}$'s, $\theta_{G_c}$'s $\leftarrow$ partition SMT($\lra$) theory $\theta$
    \ForAll{child $c$}    
    \State $I_c$ $\leftarrow$ $\PENODE(\theta_c, G_c)$
    \State $I_y^c \leftarrow \PEEDGE(\theta_{y,c}, I_c)$
    \EndFor        
    \State \textbf{Return} $I_y = \shatter\!\left(\{I_y^c \}_c\right)$
    \end{algorithmic}    
  \end{multicols}
  \label{alg: two pe algorithms}
\end{algorithm*}

\subsection{PROOF OF THEOREM~\ref{thm: complexity}}

\begin{cla}\label{cla: poly pieces and degree}
  For each path in the primal graph that starts with the root and ends with a leaf,
  and each real variable in path with height $i$, 
  its number of polynomial pieces is $O(n \cdot c^{i + 1})$.
\end{cla}
\begin{proof}
  The proof can be done by mathematical induction.
  Denote the real variable with height $i$ in the path by $x_i$
  For $i = 0$, since the number of $\lra$ literals is $c$,
  then there are at most $c$ critical points for real variable $x_0$
  and therefore there are at most $c+1$ polynomial pieces for $x_0$.
  
  Suppose that the claim holds for $i$,
  that is, the number of polynomial pieces for $x_i$ is $O(n \cdot c^{i+1})$.
  To obtain critical points for variable $x_{i+1}$,
  we collect integration bounds on variable $x_i$ whose size is $O(n \cdot c^{i+1})$ by assumption.
  Since the critical points of variable $x_{i+1}$ are obtained by solving $b_1 = b_2$ w.r.t. variable $x_{i+1}$ for $b_1, b_2$ in bounds on variable $x_i$,
  where there are at most $c$ bounds containing $x_{i+1}$ and the rest bounds are numerical ones,
  there are at most $O(n \cdot c^{i + 2})$ solutions.
  Therefore, the number of polynomial pieces for $x_{i+1}$ is $O(n \cdot c^{i+2})$,
  which finishes our proof.
\end{proof}

\begin{proof}{\textbf{ (Proof of Theorem \ref{thm: complexity})}}
  Let $p$ be an arbitrary path in the pseudo tree $T$ that starts with the root and ends with a leaf.
  Denote the maximum polynomial degree in weight functions by $d$.
  By Claim \ref{cla: poly pieces and degree}
  for each variable,
  it has at most $O(n \cdot c^{h_p})$ polynomial pieces.
  Moreover from Prop. \ref{pro: piece poly}, 
  polynomials defined over each pieces have 
  at most $n(d+h_p)$ polynomial degree.
  Therefore the set of values chosen to do instantiation on a certain real variable has size $O(n^3 \cdot c^{h_p})$
  and each path $p$ induces a search space with size $O((n^3 \cdot c^{h_p})^{h_t})$ since length of each path is bounded by $h_t$.  
  
  The pseudo tree $T$ is covered by $l$ such directed paths.
  The union of their individual search spaces covers the whole search space,
  where every distinct full path in the search space appears exactly once.
  Therefore, the size of the search space is bounded by $O(l \cdot (n^3 \cdot c^{h_p})^{h_t})$.
\end{proof}

\section{CACHING}
\label{s:caching}
Our algorithm allows caching in two sense. The first is the caching of pieces, i.e. intervals and polynomial degrees obtained from child nodes, which can be considered as constraints from child nodes.
The pieces of a certain nodes is decided both by instantiation values from its father node as well as pieces from child nodes.
Although we instantiate root nodes with distinct values, the constraints from child nodes for a certain node remains unchanged as long as they have the same father-child relation in subtree.

Another case where caching is possible is values of $p(y)$ as defined in Prop. \ref{pro: piece poly} at instantiations of variable $x$.
This is possible because for a certain node, its pieces resulting from different instantiation values of its grandfather node might intersects.
This is especially helpful when there is a long path in primal graphs and caching can save a lot computational effort.

\section{PIECE ENUMERATION ALGORITHM}
\label{sec: appendix pe}
We summarize piece enumeration algorithms for two variable theory and for theory with tree primal graphs as described in Section~\ref{sec: Piece Enum} in Algorithm~\ref{alg: two pe algorithms}. Both get\_bound\_degree and get\_degree are trivial operations for specifying integration bounds and polynomial degree. They are applied when the magnitude order of integration bounds are fixed and thus they can be done by scanning through related theories.

\end{document}